\documentclass{research}

\usepackage{mathtools}
\usepackage{amsthm}
\usepackage{algorithm}
\usepackage{algpseudocode}
\usepackage{bbm}
\usepackage{tikz}
\usepackage{adjustbox}
\usepackage{overpic}
\usepackage{rotating}
\usepackage{enumitem}
\usepackage{url}
\usepackage{wrapfig}
\usepackage{booktabs}
\usepackage{array}
\usepackage{subcaption}
\usepackage{pgfplots}
\pgfplotsset{compat=1.17}

\usepackage{mdframed}

\theoremstyle{definition}

\newtheorem{lem}{Lemma}

\newtheorem{cor}{Corollary}

\usepackage{booktabs}
\usepackage{array}
\usepackage{subcaption}
\usepackage{tikz}
\usepackage{pgfplots}
\pgfplotsset{compat=1.17}

\definecolor{gold}{rgb}{0.83, 0.69, 0.22}

\newcommand{\E}{\mathbb{E}}

\newcommand{\R}{\mathbb{R}}

\RequirePackage[authoryear]{natbib}

\definecolor{darkblue}{RGB}{40,70,140}
\definecolor{sand}{RGB}{248,245,235}
\definecolor{sandborder}{RGB}{220,215,200}
\definecolor{parchment}{RGB}{252,249,240}
\definecolor{parchmentborder}{RGB}{222,215,195}
\definecolor{sage}{RGB}{244,248,244}
\definecolor{sageborder}{RGB}{200,210,200}
\definecolor{lavender}{RGB}{247,246,250}
\definecolor{lavenderborder}{RGB}{210,205,225}

\title{Continuity Laws for Sequential Models}

\author[1]{Annan Yu}
\author[2]{Dongwei Lyu}
\author[3,4]{N. Benjamin Erichson}

\affiliation[1]{Center for Applied Mathematics, Cornell University}
\affiliation[2]{University of California, Berkeley}
\affiliation[3]{Lawrence Berkeley National Laboratory}
\affiliation[4]{International Computer Science Institute}

\abstract{%
Inductive biases influence the behavior and performance of sequential models. In this work, we study an underexplored inductive bias in sequential modeling: continuity in time. We ask a simple question: do models motivated by continuous-time formulations, such as state-space models, actually behave continuously in time, and does this translate into better performance on tasks with continuous temporal structure? To answer this, we formalize model continuity as convergence under temporal refinement, where a model is continuous if its predictions approach an underlying continuous trajectory as the temporal discretization is refined. We show that S4 exhibits stable continuous behavior, whereas S6 (the core of Mamba) can be more sensitive to input amplitude and selective dynamics, despite being derived from a continuous dynamical system. To study whether this distinction matters for learning, we also need a corresponding notion of task continuity. We therefore introduce a metric to quantify the continuity of datasets directly from their temporal structure. Across benchmarks, we find a clear empirical alignment between task continuity, model continuity, and model performance. Beyond an inductive bias, continuity also has practical consequences: we show that it enables a simple temporal subsampling strategy that improves both efficiency and performance.
}

\correspondence{Annan Yu @ \email{ay262@cornell.edu}}

\begin{document}
	
	\maketitle

\section{Introduction}

The recent surge of interest in sequential modeling, spanning language~\citep{vaswani2017attention,devlin2019bert,brown2020language}, time series~\citep{ansari2024chronos,das2023decoder}, and scientific data~\citep{li2020fourier,lu2019deeponet,li2020neural}, has produced many new architectures that achieve strong performance on specific benchmarks, often at specific scales and under carefully chosen training recipes. However, benchmark performance alone gives only a partial view of model utility. It tells us what worked in one regime, but not why it worked, when it should continue to work, or when it is likely to fail. This is not a criticism of benchmarking; benchmarks are essential, but they are most useful when paired with principles that explain what is being measured. In more classical areas such as numerical analysis, algorithms are often judged not only by their performance on test problems, but also by the properties they preserve: stability, convergence, and sensitivity to discretization~\citep{lax1956survey,trefethen2019approximation}. For sequential models, we often lack analogous principles, making it hard to know whether a model has captured coherent structure in the data, or whether it is narrowly tuned to a particular temporal regime.

Neural scaling laws~\citep{kaplan2020scaling} provide one example of such a principle: they describe how performance changes as model and data size vary. While these laws are insightful, they answer only part of the question. Scaling laws tell us what may happen as a model becomes larger or sees more data; they say less about whether the model has the right structure for the problem in the first place. Scaling a mismatched architecture may improve performance, but it does not by itself make the architecture well-matched to the structure of the task. This matters especially in scientific and engineering applications, where data and compute are often limited, and where the structure of the problem is not arbitrary. In such settings, the question is not only ``what can a sufficiently large model learn?'' but ``what does this model tend to learn?'' This brings inductive biases to the forefront. For sequential models, several such biases are already familiar~\citep{yu2025understanding,ebrahimi2026induction}, including frequency~\citep{yu2024tuning,piao2024fredformer}, recency~\citep{hegazyrecency}, locality~\citep{beltagy2020longformer}, and multi-scale structure~\citep{chung2016hierarchical}. They help explain how a model processes time, memory, and variation across scales, but they do not exhaust the ways in which a model can be matched, or mismatched, to the temporal structure of a task. In this paper, we take a new perspective by asking:

\begin{tcolorbox}[
  colback=parchment,
  colframe=parchmentborder,
  boxrule=0.4pt,
  left=10pt,
  right=10pt,
  top=6pt,
  bottom=6pt,
  arc=2pt
]
\centering
\emph{Do models motivated by continuous-time formulations actually behave continuously in time, and does this help on tasks with continuous temporal structure?}
\end{tcolorbox}

\begin{figure}[!t]
    \centering
    \includegraphics[width=0.99\linewidth]{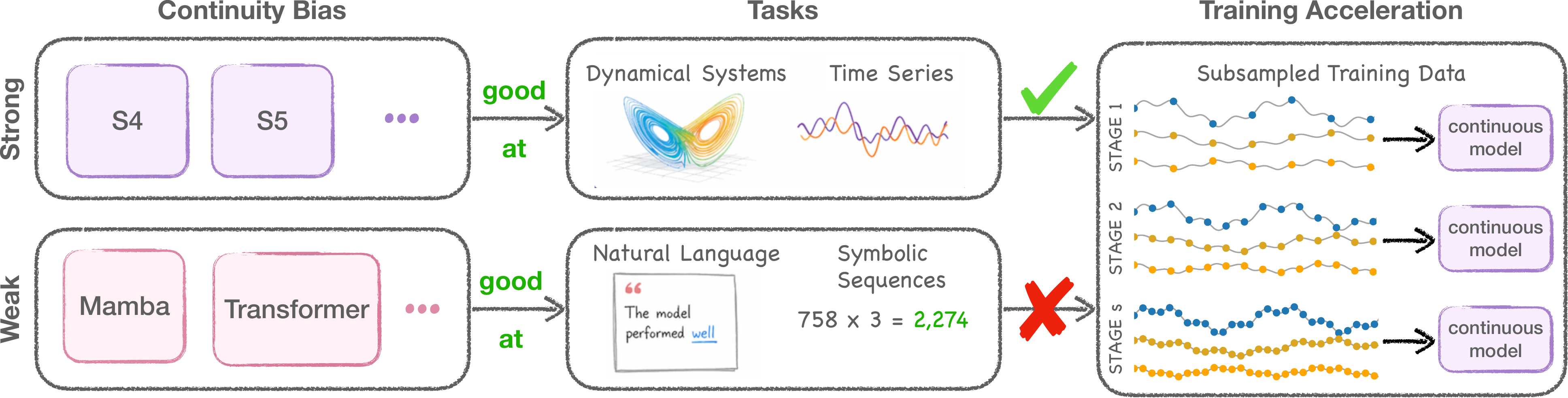}
\caption{Continuity as a principle for model--task fit. Sequential models differ in how continuously they represent time, and tasks differ in how much continuous temporal structure they contain. Models with stronger temporal continuity are better matched to continuous problems such as dynamical systems and time series, while models with weaker continuity may be better matched to more discrete problems such as text or symbolic operations. When a model and task are both continuous, this structure can also be exploited algorithmically. An example is stage-wise training, where training begins on temporally subsampled data and progressively refines the resolution.}
    \label{fig:signature}
\end{figure}

On the model side, architectures differ in how they represent time. Transformers~\citep{vaswani2017attention} operate through discrete token-to-token interactions. State-space models (SSMs) such as S4~\citep{gu2021efficiently}, S5~\citep{smith2022simplified}, and S6~\citep{gu2023mamba} take a different route: they are derived from continuous dynamical systems and are often described as continuous-time models. On the task side, it is intuitive that, for example, language is discrete and time series are continuous.
By a \textit{continuity law}, we mean a principle that characterizes the continuity of the model, the continuity of the task, and consequences when the two are aligned or misaligned. These ideas are made precise via three questions (see~\Cref{fig:signature}):
\begin{itemize}[leftmargin=*]
    \item \textbf{What is a continuous model?}
    For models with an underlying continuous-time formulation, we formalize continuity through convergence under temporal refinement: when the same underlying signal is sampled more finely, the model's discrete predictions should approach a common continuous trajectory. The speed of this convergence separates models in a way that their continuous-time motivation alone does not. S4 and S5 exhibit stable continuous behavior, whereas the convergence of S6 is very sensitive to input amplitude and selective dynamics.

    \item \textbf{What is a continuous task?}
    Sequential data have different degrees of continuous temporal structure, depending on how nearby observations relate to each other and how that relationship changes across time. Task continuity is not a binary label: we treat task continuity as a measurable property of the data rather than as a semantic label such as ``language'' or ``time series.'' We introduce a metric that quantifies how strongly nearby elements in a sequence resemble one another relative to far-apart elements, placing datasets on a continuum from more discrete to more continuous.

    \item \textbf{How do model and task continuity interact?}
    Once both notions are defined, we can ask whether they align. Across controlled dynamical systems and Long-Range Arena (LRA) benchmarks, we find that models with stronger continuity perform better as tasks become more continuous. This identifies continuity as an inductive bias for reasoning about model--task fit.
\end{itemize}

These notions treat continuity as an explanatory principle to help characterize models, measure tasks, and predict when the two are well matched. However, continuity can also be used constructively. As a concrete example, we revisit the state-space structure of SSMs. Prior work on efficient SSMs has focused on reducing the state dimension, thereby compressing the hidden dynamics~\citep{padhy2026aire,chahine2025curious,gwak2024layer}. We take a complementary view: \textit{if a model behaves continuously in time, then one can also compress the time axis}. This leads to a stage-wise training strategy that begins with temporally subsampled inputs and progressively refines the resolution. While this is not the main focus of the paper, it illustrates a broader point: continuity is not only a diagnostic property, but a structure that can be exploited.

See Appendix~\ref{app:relatedwork} for related work. We summarize the main contributions of this paper as follows.

\begin{enumerate}[leftmargin=*]
    \item We introduce a theoretical framework for assessing continuity in SSM-like models via temporal refinement. This framework shows that a continuous-time origin does not by itself guarantee continuous behavior after discretization: S4 and S5 remain well behaved under refinement, but S6 is considerably more affected by input magnitude and selective dynamics. We support our conclusion with numerical studies and experiments on $4$ classes of dynamical systems.

    \item We propose a data-driven metric for task continuity based on local temporal similarity. By comparing nearby elements in a sequence to far-apart ones, the metric places datasets along a spectrum from discrete to continuous. On both controlled dynamical systems and Long-Range Arena benchmarks~\citep{tay2020long}, the resulting scores agree well with the intuitive continuity of the tasks.

    \item We empirically demonstrate a strong match between model continuity and task continuity. Across S4, S6, and Transformer architectures, we observe that models with stronger temporal continuity become increasingly favorable as the task itself becomes more continuous. These results support continuity as a useful inductive bias for understanding model--task compatibility.

    \item We show that continuity is useful not only for analysis, but also for algorithm design. Motivated by the temporal stability of S4, we invent a stage-wise training method that starts from temporally coarsened inputs and gradually restores full resolution. On five of the six LRA tasks, we achieve \textit{any} reachable accuracy in less wall-clock time, and in many cases even improve the final accuracy.
\end{enumerate}

\section{A continuous framework for state-space models}

Suppose $\mathbf{u}_0, \ldots, \mathbf{u}_{L-1} \in \R^d$ is a sequence of input vectors. A sequential module maps this input to an output sequence $\mathbf{y}_0, \ldots, \mathbf{y}_{L-1} \in \R^d$. For example, a (single-head) attention layer computes
\[
    \mathbf{y}_k = \textcolor{gold}{\mathbf{W}_V} \mathbf{U} \;\text{softmax}\!\left(\mathbf{u}_k^\top \textcolor{gold}{\mathbf{W}_Q}^\top \textcolor{gold}{\mathbf{W}_K} \mathbf{U} / \sqrt{d}\right)^\top, \qquad 
    \mathbf{U} = \begin{bmatrix} \mathbf{u}_0 & \cdots & \mathbf{u}_{L-1} \end{bmatrix} \in \R^{d \times L},
\]
where we use gold to highlight the trainable parameters. The output is a weighted combination of $\mathbf{W}_V \mathbf{u}_0, \ldots, \mathbf{W}_V \mathbf{u}_{L-1}$, and the entire computation is inherently discrete. For simplicity of exposition, we assume $d = 1$ from now on; extending the discussion to higher dimensions is straightforward. SSMs take a different way: they begin with a continuous-time dynamical system:
\begin{equation}\label{eq.contdyn}
    \begin{aligned}
        \mathbf{x}'(t) = \Delta(u(t)) \bigl(\mathbf{A} \mathbf{x}(t) + \mathbf{B}(u(t)) u(t)\bigr), \quad y(t) = \mathbf{C}(u(t)) \mathbf{x}(t) + D u(t),
    \end{aligned}
    \qquad \mathbf{x}(0) = \boldsymbol{0}.
\end{equation}
Here $\mathbf{A} \in \R^{n \times n}$, $\mathbf{B}: \R \to \R^{n \times 1}$, $\mathbf{C}: \R \to \R^{1 \times n}$, $D \in \R$, and $\Delta: \R \to (0,\infty)$. 
To connect this continuous system to discrete data, we view $u_0, \ldots, u_{L-1}$ as samples of an underlying continuous signal $u(t)$. The output sequence is then obtained by discretizing the dynamics, for instance using zero-order hold (ZOH) or bilinear discretization:
\begin{equation}\label{eq.discretizeLTI}
    \begin{aligned}
        &\textbf{(ZOH)} \qquad 
        \overline{\mathbf{A}}_k =  \exp( \Delta(u_k) \mathbf{A}), \quad 
        \overline{\mathbf{B}}_k = \mathbf{A}^{-1} (\overline{\mathbf{A}}_k - \mathbf{I}_n) \mathbf{B}(u_k), \quad 
        \overline{\mathbf{C}}_k = \mathbf{C}(u_k), \\
        &\textbf{(Bilinear)} \qquad 
        \begin{aligned}
            \overline{\mathbf{A}}_k &=  \left(\mathbf{I}_n - (1/2)\Delta(u_k) \mathbf{A}\right)^{-1} \left(\mathbf{I}_n + (1/2)\Delta(u_k) \mathbf{A}\right), \\
            \overline{\mathbf{B}}_k &= \left(\mathbf{I}_n - (1/2)\Delta(u_k) \mathbf{A}\right)^{-1}  \Delta(u_k) \mathbf{B}(u_k), \quad 
            \overline{\mathbf{C}}_k = \mathbf{C}(u_k),
        \end{aligned}
    \end{aligned}
\end{equation}
followed by the discrete recursion
\begin{equation}\label{eq.disdyn}
    \begin{aligned}
        \mathbf{x}_{k+1} = \overline{\mathbf{A}}_k \mathbf{x}_k + \overline{\mathbf{B}}_k u_k, \quad y_k = \overline{\mathbf{C}}_k \mathbf{x}_k + D u_k,
    \end{aligned}
    \qquad \mathbf{x}_0 = \boldsymbol{0}.
\end{equation}

The key difference between S4 and S6 lies in how $\mathbf{B}$, $\mathbf{C}$, and $\Delta$ are defined:
\begin{equation}\label{eq.S4}
\begin{aligned}
        &\textbf{(S4)} \qquad 
    \mathbf{B}(u) \equiv \textcolor{gold}{\mathbf{B}},\quad 
    \mathbf{C}(u) \equiv \textcolor{gold}{\mathbf{C}}, \quad 
    \Delta(u) \equiv \textcolor{gold}{\Delta}, \\
    &\textbf{(S6)} \qquad 
    \mathbf{B}(u) = \textcolor{gold}{\mathbf{B}}u,\quad 
    \mathbf{C}(u) = \textcolor{gold}{\mathbf{C}}u, \quad 
    \Delta(u) = \text{softplus}(\textcolor{gold}{{w}_\Delta} u + \textcolor{gold}{b_\Delta}).
\end{aligned}
\end{equation}
In S4, they are constant, resulting in a linear time-invariant system. In contrast, S6 introduces input dependency, leading to a nonlinear, time-varying system.

\section{How continuous is a continuous model?}\label{sec:contmodel}

\begin{wrapfigure}{r}{0.5\textwidth}
  \begin{center}
  \vspace{-1\baselineskip}
    \includegraphics[width=0.45\textwidth]{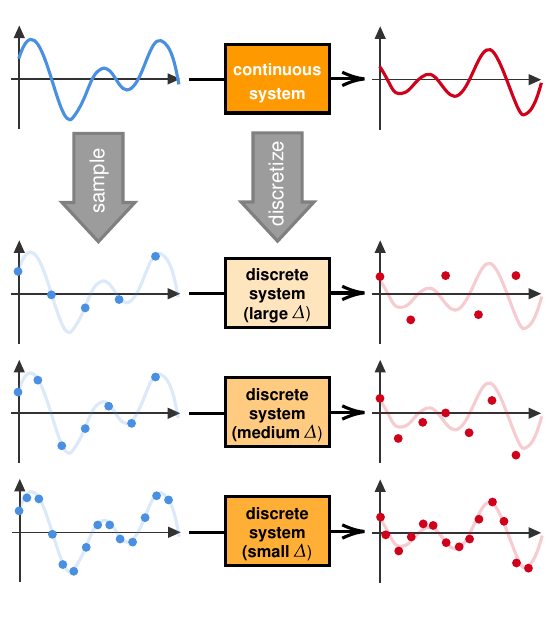}
    \vspace{-2\baselineskip}
  \end{center}
  \caption{We assess a model's continuity by measuring how quickly the discrete output converges to the continuous one under temporal refinement.}
  \label{fig:thmillustrate}
\end{wrapfigure}

SSMs are often described as continuous-time models, but how accurate is this description in practice? To answer this, we need a concrete notion of continuity. A natural starting point is the following: the discrete outputs $y_0, \ldots, y_{L-1}$ should closely match the continuous trajectory $y(t)$ evaluated at the same timestamps. However, this raises an immediate issue: in practice, we are only given a discrete sequence $u_0, \ldots, u_{L-1}$, with no guarantee that it arises from a meaningful continuous signal $u(t)$. Without such a signal, the notion of a ``continuous output'' $y(t)$ is ill-defined.

To make progress, we adopt a slightly different viewpoint (see~\Cref{fig:thmillustrate}). For now, we assume access to an underlying continuous input $u: [0,1] \to \R^d$, which induces a continuous output $y(t)$. We then sample $u$ on a uniform grid with step size $\tau > 0$:
\[
    u(0), u(\tau), \ldots, u(\lfloor 1/\tau \rfloor \tau).
\]
This produces a discrete input sequence, from which the model generates outputs.

We say that an SSM behaves continuously if, under temporal refinement, i.e., as $\tau \to 0^+$, the discrete outputs converge to the continuous trajectory at a fast rate. In other words, finer sampling of the same underlying signal should lead to more accurate approximations of the continuous solution. There is only one subtlety that remains: the discretization rules in~\cref{eq.discretizeLTI} implicitly assume a unit time step and do not account for the refinement parameter $\tau$. To properly capture the effect of temporal refinement, we modify the discretization as follows:
\begin{equation}\label{eq.discretizeLTItau}
    \begin{aligned}
        &\textbf{(ZOH)} \qquad 
        \overline{\mathbf{A}}^\tau_k =  \exp(\tau \Delta(u(\tau k)) \mathbf{A}), \;\; 
        \overline{\mathbf{B}}^\tau_k = \mathbf{A}^{-1} (\overline{\mathbf{A}}^\tau_k \!-\! \mathbf{I}_n) \mathbf{B}(u(\tau k)), \;\; 
        \overline{\mathbf{C}}^\tau_k = \mathbf{C}(u(\tau k)), \\
        &\textbf{(Bilinear)} \qquad 
        \begin{aligned}
            \overline{\mathbf{A}}^\tau_k &=  \left(\mathbf{I}_n - \frac{\tau \Delta(u(\tau k))}{2} \mathbf{A}\right)^{-1} \left(\mathbf{I}_n + \frac{\tau \Delta(u(\tau k))}{2} \mathbf{A}\right), \\
            \overline{\mathbf{B}}^\tau_k &= \left(\mathbf{I}_n - \frac{\tau \Delta(u(\tau k))}{2} \mathbf{A}\right)^{-1} \tau \Delta(u(\tau k)) \mathbf{B}(u(\tau k)), \quad 
            \overline{\mathbf{C}}^\tau_k = \mathbf{C}({u}(\tau k)).
        \end{aligned}
    \end{aligned}
\end{equation}

The only change is to replace $u(k)$ with $u(\tau k)$ and $\Delta$ with $\tau \Delta$. Using this notion of continuity, we theoretically quantify the convergence rate of the continuous system in the general form of~\cref{eq.contdyn}.

\begin{lem}\label{lem.discretize}
    Let $u(t): [0,1] \rightarrow \R$ be $L_u$-Lipschitz continuous. Let $\mathbf{A} \in \R^{n \times n} $, $\mathbf{B}: \R \rightarrow \R^{n \times 1}, \mathbf{C}: \R \rightarrow \R^{1 \times n}$, ${D} \in \R$, and $\Delta: \R \rightarrow (0, \infty)$ be given, where $\mathbf{B}$, $\mathbf{C}$, and $\Delta$ are $L_B$-, $L_C$-, and $L_\Delta$-Lipschitz continuous (in $\|\cdot\|_2$-norm), respectively. Let
    \[
	M_u \!=\! \sup_{t \in [0,1]}\|u(t)\|_2, \;\;
	M_B \!=\! \sup_{t \in [0,1]}\|\mathbf{B}(u(t))\|_2,\;\;
	M_C \!=\! \sup_{t \in [0,1]}\|\mathbf{C}(u(t))\|_2,\;\;
	M_\Delta \!=\! \sup_{t \in [0,1]} \Delta(u(t))
    \]
    be the maximum moduli of $\mathbf{B}, \mathbf{C}$, and $\Delta$ achieved by $u(t)$, respectively.
    Consider the output $y(t): [0,1] \rightarrow \R$ defined by the continuous system in~\cref{eq.contdyn}. For any $0 < \tau < 1$, let $\mathbf{y}^\tau = (y_0^\tau, \ldots, y_{\lfloor 1 / \tau \rfloor}^\tau)$ be the output of the discrete system
    \[
    	\begin{aligned}
		\mathbf{x}_{k+1}^\tau = \overline{\mathbf{A}}^\tau_k \mathbf{x}^\tau_k + \overline{\mathbf{B}}^\tau_k u(\tau k), \quad y_k^\tau = \overline{\mathbf{C}}^\tau_k \mathbf{x}_{k}^\tau + \overline{D} u(\tau k),
	\end{aligned} \qquad\qquad \mathbf{x}^\tau_0 = \boldsymbol{0},
    \]
    where $\overline{\mathbf{A}}^\tau_k, \overline{\mathbf{B}}^\tau_k,$ and $\overline{\mathbf{C}}^\tau_k$ are defined by either the ZOH or the bilinear formulas in~\cref{eq.discretizeLTItau}.
    Then, for every $1 \leq k \leq \lfloor 1 / \tau \rfloor$, we have
    \begin{equation*}
    	|y_k^\tau \!-\! y(\tau k)| \leq M_C \frac{M_\Delta (L_B M_u \!+\! M_B) L_u \!+\! L_\Delta L_u M_B M_u e^{M_\Delta \|\mathbf{A}\|_2}}{M_\Delta \|\mathbf{A}\|_2}  (\text{exp}(M_\Delta \|\mathbf{A}\|_2) \!-\! 1) \tau + \mathcal{O}(\tau^2).
    \end{equation*}
    Moreover, if $\overline{\mathbf{A}}^\tau_k, \overline{\mathbf{B}}^\tau_k,$ and $\overline{\mathbf{C}}^\tau_k$ are obtained via the ZOH formula, then the $\mathcal{O}(\tau^2)$ term vanishes.
\end{lem}

See Appendix~\ref{app:proofs} for a proof of the lemma and the corollaries presented in the rest of this section. \Cref{lem.discretize} shows that the discrete output converges to the continuous output at a linear rate $\mathcal{O}(\tau)$. This is consistent with the classical behavior of standard discretizations such as ZOH and bilinear under smooth dynamics~\citep{oppenheim2013}. What we care about, however, is the constant in the $\mathcal{O}$-notation for this linear convergence rate. The following two corollaries are immediate consequences of~\Cref{lem.discretize}.

\begin{cor}\label{cor:S4discrete}
Suppose $\mathbf{B}, \mathbf{C}, D$, and $\Delta$ are defined by an S4 system (see~\cref{eq.S4}). Then, under the setting of~\Cref{lem.discretize}, we have
\[
	|y_k^\tau - y(\tau k)| \leq \frac{\|\mathbf{C}\|_2 \|\mathbf{B}\|_2 L_u (\text{exp}(\Delta \|\mathbf{A}\|_2) - 1)}{\|\mathbf{A}\|_2} \tau + \mathcal{O}(\tau^2).
\]
\end{cor}

\begin{cor}\label{cor:S6discrete}
Suppose $\mathbf{B}, \mathbf{C}, D$, and $\Delta$ are defined by an S6 system (see~\cref{eq.S4}). Then, under the setting of~\Cref{lem.discretize}, we have
\[
	|y_k^\tau - y(\tau k)| \leq \frac{
\|\mathbf W_{\mathbf C}\|_2 \|\mathbf W_{\mathbf B}\|_2 M_u^2 L_u
\left(2M_\Delta + |w_\Delta|M_u e^{M_\Delta\|\mathbf A\|_2}\right)
}{
M_\Delta\|\mathbf A\|_2
}
\left(e^{M_\Delta\|\mathbf A\|_2}-1\right) \tau  + \mathcal{O}(\tau^2),
\]
where
\[
M_\Delta=\text{softplus}(|w_\Delta|M_u+b_\Delta).
\]
\end{cor}

Corollary~\ref{cor:S4discrete} and~\ref{cor:S6discrete} show that both S4 and S6 converge at first order in the refinement scale $\tau$, but the key difference lies in the prefactor. For S4, the continuity constant depends only on the fixed system parameters $(\mathbf A,\mathbf B,\mathbf C,\Delta)$ and the input regularity $L_u$. In particular, once the system is fixed, the dependence on the input enters only through its temporal smoothness.
By contrast, for S6, the constant is substantially more sensitive: in addition to $L_u$, it depends on the input amplitude $M_u$, on the magnitudes of the projections $\|\mathbf W_{\mathbf B}\|_2$ and $\|\mathbf W_{\mathbf C}\|_2$, and on the selectivity parameter $|w_\Delta|$.
Hence, the continuity breaks down for large-amplitude inputs or strongly selective delta-dynamics.

\begin{figure}[!htb]
    \centering
    \begin{minipage}{0.6\textwidth}
        \begin{overpic}[width=1\textwidth]{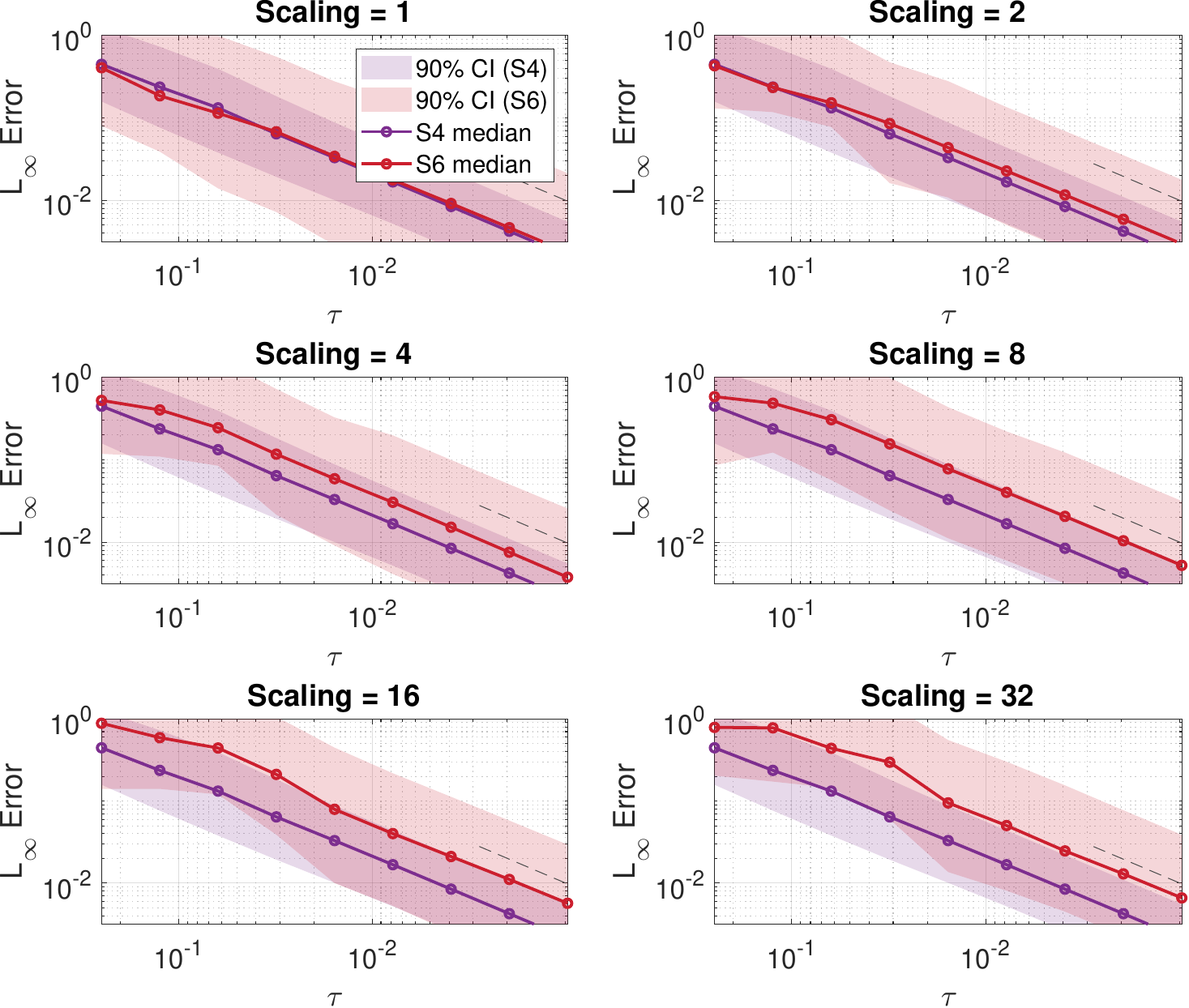}
            
        \end{overpic}
    \end{minipage}
    \hfill
    \begin{minipage}{0.3\textwidth}
        \begin{overpic}[width=1\textwidth]{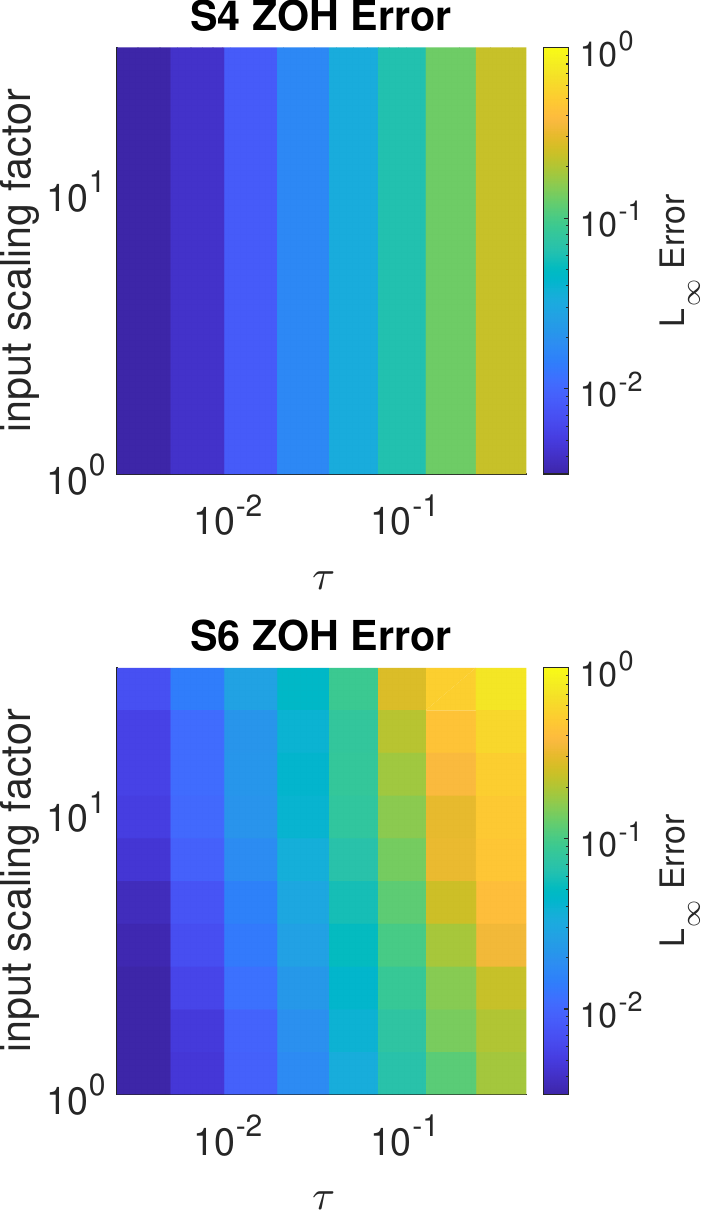}
            
        \end{overpic}
    \end{minipage}
    \caption{\textbf{Left:} Relative maximum error $\max_{1 \leq k \leq \lfloor 1 / \tau \rfloor} |y_k^\tau - y(\tau k)| \big/ \max_{1 \leq k \leq \lfloor 1 / \tau \rfloor} |y(\tau k)|$ as $\tau \to 0^+$, evaluated over $20$ randomly sampled system--input pairs. The input is scaled by factors ranging from $1$ to $32$. The gray reference line indicates first-order (linear) convergence. \textbf{Right:} Heatmap of the same relative error. For S4, the error remains stable as the input is scaled, whereas for S6, it increases significantly with input magnitude.}
    \label{fig:numericaldiscretize}
\end{figure}

To illustrate these phenomena numerically, we sample $20$ random S4 and S6 systems, together with $20$ random input functions. For each system--input pair, we discretize the input using a range of step sizes $\tau \in \{2^{-10}, \ldots, 2^{-2}\}$ and compare the resulting outputs $\mathbf{y}^\tau$ to the corresponding continuous trajectory $y(\cdot)$. The latter is approximated using a much finer discretization ($\tau = 2^{-14}$) together with a high-accuracy Runge--Kutta solver. We further probe the sensitivity to input amplitude by scaling each input function by a factor greater than $1$ and repeating the experiment. While this scaling also increases the Lipschitz constant $L_u$, we report the relative error $\max_{1 \leq k \leq \lfloor 1 / \tau \rfloor} |y_k^\tau - y(\tau k)| \big/ \max_{1 \leq k \leq \lfloor 1 / \tau \rfloor} |y(\tau k)|$ to normalize for this effect.
As shown in~\Cref{fig:numericaldiscretize}, the relative error of S4 remains stable under input scaling, whereas the error of S6 increases rapidly. This empirical behavior aligns with our theoretical predictions and highlights the fundamentally different continuity properties of the two models.

\paragraph{Empirical evidence for S6's discontinuity.} Our comparison of S4 and S6 so far is primarily theoretical, or at best numerical. You may wonder: do these findings actually matter in practice? To answer this, we study a downstream task based on dynamical systems data. In this setting, each sample is generated from an underlying parameterized dynamical system, and the goal is to recover its parameters from observed trajectories. To cover a broad range of behaviors, we consider four representative types of systems (see Appendix~\ref{app:sec3exp} for more details.): 
\begin{itemize}[leftmargin=*]
\item a linear system (damped harmonic oscillator),
\item a nonlinear system (Van der Pol oscillator),
\item a stochastic system (Ornstein--Uhlenbeck process), and
\item a chaotic system (forced Duffing oscillator).
\end{itemize}

\begin{figure}[!htb]
    \centering
    \begin{overpic}[width=1\textwidth]{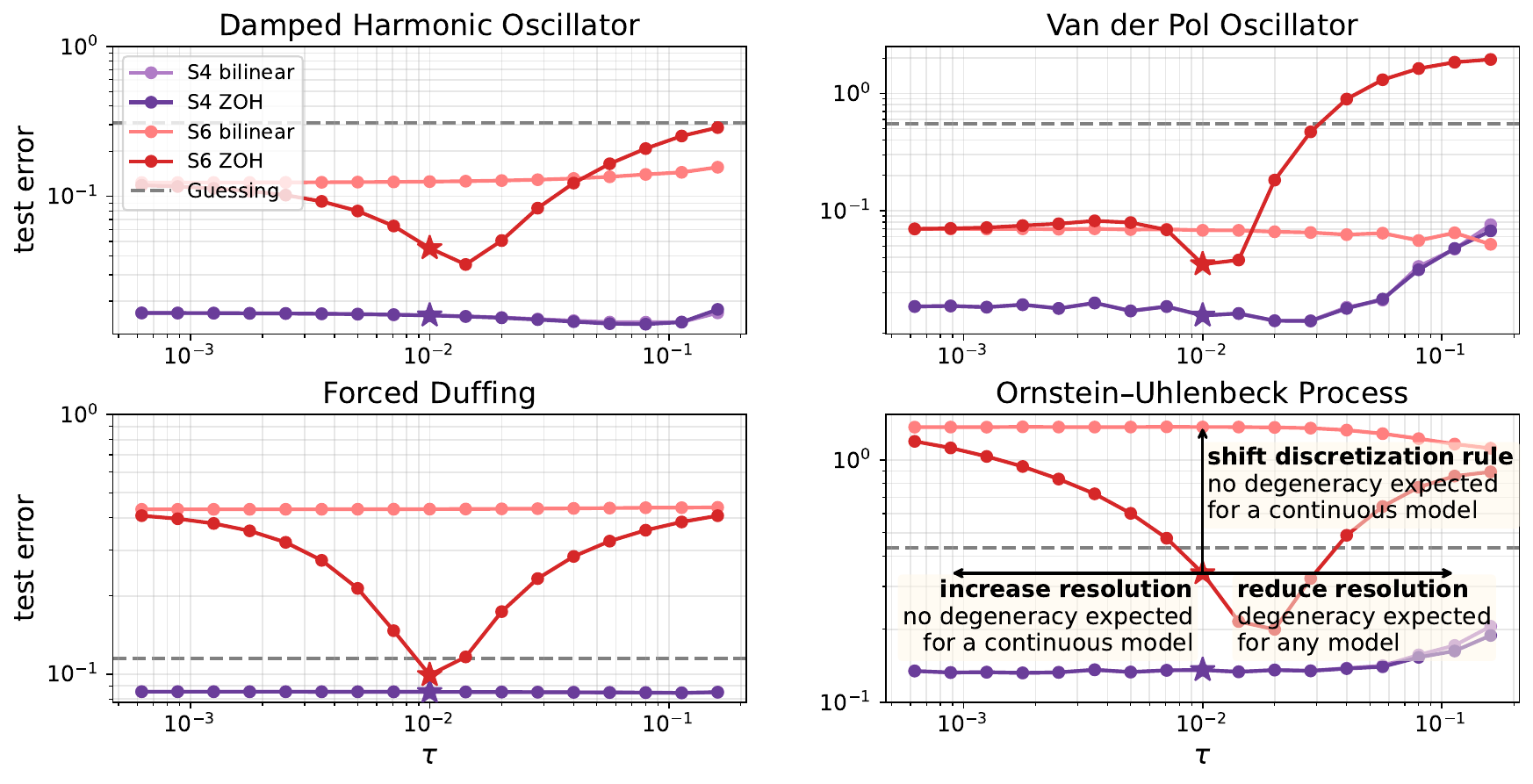}
            
    \end{overpic}
    \caption{Test error as a function of the inference-time sampling rate $\tau$ across four dynamical systems. For each model, we evaluate both ZOH and bilinear discretizations. The star marks the resolution $\tau_0 = 10^{-2}$ and the ZOH discretization used for training, and the dashed line indicates the expected error of a guessing baseline. S4 is stable across changes in $\tau$ and discretization, but S6 is sensitive to both, with performance degrading under refinement and when switching discretization schemes.}
    \label{fig:DSresample}
\end{figure}

We construct the training data by sampling these systems at a fixed rate $\tau_0 = 10^{-2}$ and, during training, adopt the ZOH discretization. Our emphasis here is not that S4 generally outperforms S6; instead, we compare a learned model against itself under two types of distribution shifts:

\begin{enumerate}[leftmargin=*]
    \item \textbf{Changing resolution.} Still using ZOH discretization, we resample the same systems at a different rate $\tau$ at inference time, and adjust the $\Delta$ parameter in the S4 or S6 model by a factor of $\tau / \tau_0$ to account for the change in resolution. Intuitively, if a model has learned the underlying continuous dynamics, then refining the sampling (i.e., $\tau < \tau_0$) should not degrade its performance (see also~\citet{queiruga2020continuous}). As shown in~\Cref{fig:DSresample}, this holds for S4: its performance remains stable under temporal refinement. In contrast, S6 degrades as the resolution increases, indicating that its discrete outputs do not faithfully track the underlying continuous dynamics.

    \item \textbf{Changing discretization rule.} We also vary the discretization method at inference time, switching from ZOH to bilinear. This provides a complementary test: if a model truly captures the continuous dynamics, its predictions should be largely invariant to the choice of discretization. As shown in~\Cref{fig:DSresample}, S4 is robust to this change, with little difference between the two schemes. In contrast, S6 exhibits noticeable degradation when switching to bilinear discretization, even at the same sampling rate $\tau = \tau_0$ used for training.
\end{enumerate}

\section{Are continuous models better at learning continuous tasks?}\label{sec:contdata}

So far, we have shown that certain models (e.g., S4) more faithfully capture underlying continuous dynamics than others (e.g., S6). Intuitively, such models should be better suited for continuous tasks. But does this intuition actually hold in practice? In this section, we first design a controlled experiment to test this hypothesis, and then move toward formalizing the notion of continuous tasks so that we can examine this effect across a broader range of problems. Throughout, we focus on three representative model classes: S4 (continuous formulation and behavior), S6 (continuous formulation but not in behavior), and Transformers (purely discrete formulation).

To begin, we construct a controlled setting in which we can gradually vary the ``degree of continuity'' of a task. At first glance, this might seem straightforward: one could simply compare performance across tasks with different levels of continuity. However, this approach introduces two key challenges:
\begin{enumerate}[leftmargin=*]
    \item \textbf{Control over difficulty.} Tasks differ in intrinsic difficulty. That a model is better at a ``continuous'' task than a ``discrete'' one does not necessarily indicate an advantage in handling continuity.
    \item \textbf{Control over other inductive biases.} Models and tasks may also differ along other dimensions, such as frequency content or long-range dependency. Without careful control, other inductive biases of models can confound any conclusions about continuity.
\end{enumerate}
To address the first issue, we compare the \emph{relative performance} of S4, S6, and Transformers as the task transitions from discrete to continuous. To address the second, we base all experiments on the same underlying dynamical systems, thereby isolating continuity as the primary varying factor. Concretely, we consider a forecasting problem on trajectories generated from the four dynamical systems described in~\cref{sec:contmodel}. Given a normalized trajectory $u: [0,1] \rightarrow [-1,1]$, we construct two representations in $\R^{16}$:
\[
    \mathbf{U}_{\text{cts}}(t) = u(t)\mathbf{v}, \qquad 
    \mathbf{U}_{\text{dct}}(t) = \sum_{i=1}^{16} \mathbbm{1}_{\{u(t) \in R_i\}} \mathbf{q}_i,
\]
where $\mathbf{v} \in \R^{16}$ is a fixed random unit vector, $\{\mathbf{q}_i\}_{i=1}^{16}$ form an orthonormal basis in $\R^{16}$, and $\{R_i\}_{i=1}^{16}$ partition $[-1,1]$ into equal intervals. The first embedding preserves the continuity of the signal, while the second introduces quantization and produces a more discrete representation (see~\citet{yu2025understanding} for more discussion). We then form models' input by interpolating between these two regimes by defining
\begin{equation}\label{eq.mixcontdis}
    \mathbf{U}_\eta(t) = \eta \mathbf{U}_{\text{cts}}(t) + (1-\eta)\mathbf{U}_{\text{dct}}(t), \qquad \eta \in [0,1].
\end{equation}
Therefore, as $\eta$ increases from $0$ to $1$, the task transitions smoothly from discrete to continuous.

\begin{figure}[!htb]
    \centering
    \begin{overpic}[width=1\textwidth]{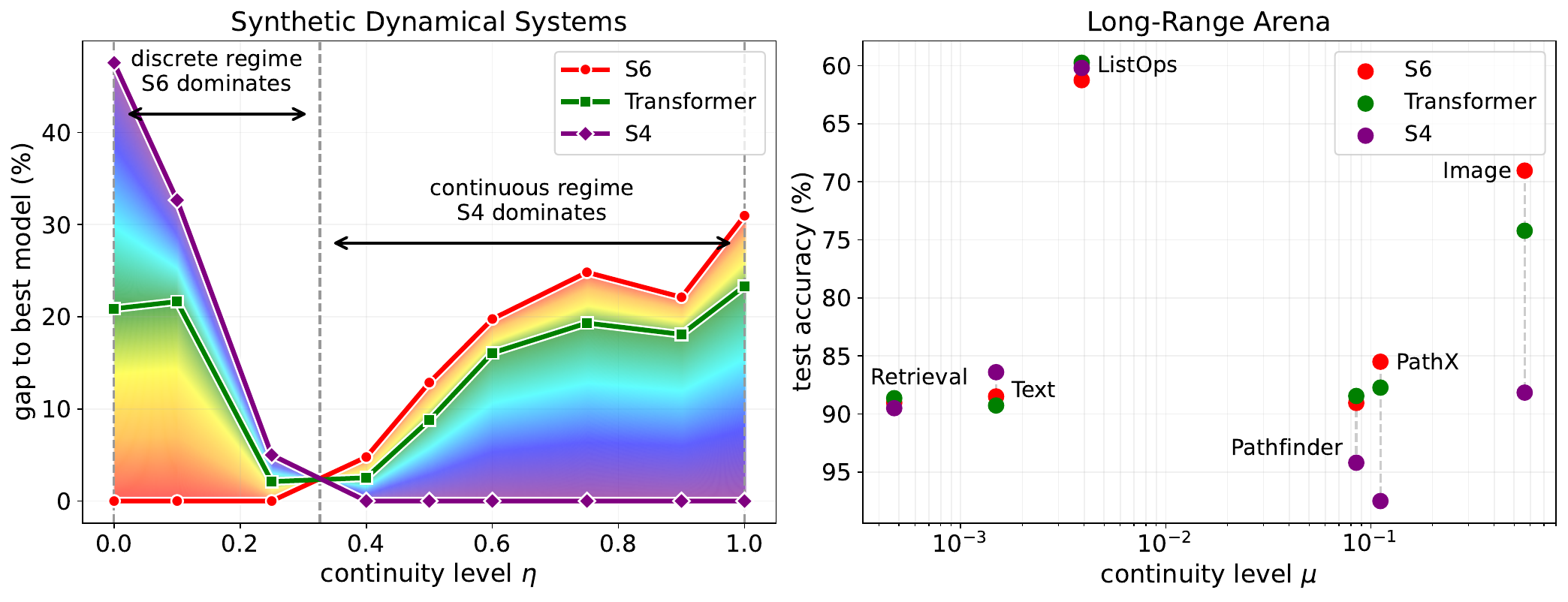}
            
    \end{overpic}
    \caption{\textbf{Left:} Relative $L_\infty$ forecasting error gap to the best model, defined as $\mathrm{gap}(\mathcal{M})=(e_\mathcal{M}-e_{\min})/e_{\min}$, where $e_{\min}=\min_{\mathcal{M}'} e_{\mathcal{M}'}$ and $e_{\mathcal{M}}$ is the $L_\infty$ forecasting error of model $\mathcal{M}$, for S4, S6, and Transformers on dynamical systems data as the underlying representation becomes more continuous. \textbf{Right:} Computed continuity score $\mu = \sum_{t=1}^{16} \mu_t / 16$ for the six LRA tasks, together with the reported accuracies of S4, S6, and Transformers. In both panels, S4 gradually overtakes S6 and Transformers as the tasks become more continuous. See Appendix~\ref{app:sec4exp} for more details.}
    \label{fig:continuoustasks}
\end{figure}

For each fixed continuity level $\eta$, we train an S4, an S6, and a Transformer to forecast the dynamical systems. Results in~\Cref{fig:continuoustasks} (Left) reveal a clear pattern: on more discrete tasks (i.e., when $\eta$ is small), S4 underperforms S6 and Transformers, but as the task becomes more continuous, S4 begins to outperform both. Importantly, all tasks are derived from the same underlying dynamical systems; only the representation changes. This makes clear that continuity, by itself, can substantially alter how these models perform.

\begin{table}[!htb]
\centering
\caption{Normalized relative gaps in the dynamical system forecasting error across continuity levels $\eta$, with cell colors computed from row-normalized values so that each model has comparable total mass across the row while preserving within-row proportions. We see that the better predictions happen at the lower-left part of the table, where both the model and the problem are discrete, and at the upper-right part of the table, where both the model and the problem are continuous.}

\setlength{\tabcolsep}{5pt}
\renewcommand{\arraystretch}{1.15}
\begin{tabular}{lccccccccc}
\toprule
 & 0 & 0.1 & 0.25 & 0.4 & 0.5 & 0.6 & 0.75 & 0.9 & 1 \\
\midrule
S4
& \cellcolor[HTML]{F8696B}0.630 & \cellcolor[HTML]{EF6E6C}0.560 & \cellcolor[HTML]{DA7A6E}0.418 & \cellcolor[HTML]{9D9D75}0.112 & \cellcolor[HTML]{63BE7B}0.000 & \cellcolor[HTML]{63BE7B}0.000 & \cellcolor[HTML]{63BE7B}0.000 & \cellcolor[HTML]{63BE7B}0.000 & \cellcolor[HTML]{63BE7B}0.000 \\
DeltaNet
& \cellcolor[HTML]{E5746D}0.494 & \cellcolor[HTML]{DF776E}0.450 & \cellcolor[HTML]{C18971}0.271 & \cellcolor[HTML]{9D9D75}0.114 & \cellcolor[HTML]{81AD78}0.033 & \cellcolor[HTML]{97A075}0.093 & \cellcolor[HTML]{999F75}0.098 & \cellcolor[HTML]{94A276}0.083 & \cellcolor[HTML]{94A276}0.083 \\
Transformer
& \cellcolor[HTML]{C68670}0.299 & \cellcolor[HTML]{D27F6F}0.368 & \cellcolor[HTML]{C98470}0.319 & \cellcolor[HTML]{9D9D75}0.114 & \cellcolor[HTML]{8DA676}0.064 & \cellcolor[HTML]{9E9C75}0.116 & \cellcolor[HTML]{A49974}0.140 & \cellcolor[HTML]{A29A74}0.131 & \cellcolor[HTML]{AB9573}0.169 \\
LinAttn
& \cellcolor[HTML]{7EAE78}0.029 & \cellcolor[HTML]{B88E72}0.225 & \cellcolor[HTML]{BE8A71}0.256 & \cellcolor[HTML]{A99673}0.161 & \cellcolor[HTML]{BF8A71}0.261 & \cellcolor[HTML]{C38771}0.283 & \cellcolor[HTML]{A19B74}0.127 & \cellcolor[HTML]{AB9573}0.169 & \cellcolor[HTML]{B49072}0.209 \\
S6
& \cellcolor[HTML]{A49974}0.137 & \cellcolor[HTML]{B19173}0.196 & \cellcolor[HTML]{CC8270}0.335 & \cellcolor[HTML]{A79774}0.150 & \cellcolor[HTML]{9B9E75}0.105 & \cellcolor[HTML]{A99673}0.161 & \cellcolor[HTML]{B39072}0.203 & \cellcolor[HTML]{AE9373}0.181 & \cellcolor[HTML]{BD8B71}0.253 \\
Mamba3
& \cellcolor[HTML]{63BE7B}0.000 & \cellcolor[HTML]{63BE7B}0.000 & \cellcolor[HTML]{63BE7B}0.000 & \cellcolor[HTML]{63BE7B}0.000 & \cellcolor[HTML]{C38771}0.281 & \cellcolor[HTML]{BD8A71}0.254 & \cellcolor[HTML]{D57D6F}0.387 & \cellcolor[HTML]{BC8B71}0.245 & \cellcolor[HTML]{EE6F6C}0.553 \\
\bottomrule
\end{tabular}
\label{tab:gap_heatmap}
\end{table}

In~\Cref{tab:gap_heatmap}, we show further results with three additional models: Linear Attention~\citep{katharopoulos2020transformers}, DeltaNet~\citep{yang2024parallelizing}, and Mamba-3~\citep{lahoti2026mamba}. We report the ``normalized relative gap,'' a normalized version of the relative gap demonstrated in~\Cref{fig:continuoustasks}. See Appendix~\ref{app:sec4exp} for a detailed definition and motivation.

While the controlled experiment reveals the advantage of continuous models on continuous tasks (and vice versa), most real-world problems do not admit a representation such as~\cref{eq.mixcontdis}. To study continuity in a broader setting, we introduce a metric that can be applied to a general sequential dataset. At a high level, we measure how similar nearby elements in a sequence are, compared to pairs that are far apart. A dataset with stronger continuity should exhibit higher local similarity that decays gradually with temporal separation. To make this precise, let $\mathbf{u} = (\mathbf{u}_0,\ldots,\mathbf{u}_{L-1})$ denote a sequence drawn from the dataset distribution $\mathcal{D}$, where each $\mathbf{u}_k \in \R^d$. Let $K:\R^d \times \R^d \to [0,1]$ be a similarity kernel, e.g, the cosine kernel $K(\mathbf{x},\mathbf{y}) = \left(1 + \mathbf{x}^\top \mathbf{y} / \|\mathbf{x}\|_2\,\|\mathbf{y}\|_2 \right) / 2$. For a lag parameter $t \geq 1$ and a fixed gap $\varrho \gg t$, we define
\[
    \mu_t
    \!:=\!
    \frac{
        \E_{\mathbf{u} \sim \mathcal{D}}
        \E_{\,k \sim \text{Unif}(0:L-t-1)}
        \bigl[K(\mathbf{u}_k,\mathbf{u}_{k+t})\bigr]
        \!-\! \beta
    }{
        \E_{\mathbf{u} \sim \mathcal{D}}
        \E_{\,k \sim \text{Unif}(0:L-1)}
        \bigl[K(\mathbf{u}_k,\mathbf{u}_k)\bigr]
        \!-\! \beta   
    }, \quad \beta \!=\! \E_{\mathbf{u} \sim \mathcal{D}}
        \E_{(k,k’) \sim \text{Unif}(\{|k - k’| > \varrho\})}
        \bigl[K(\mathbf{u}_k,\mathbf{u}_{k’})\bigr].
\]
In practice, we approximate these expectations by averaging over sequences in a dataset, uniformly sampled valid positions $k$, and randomly sampled far-apart pairs $(k,k’)$. For an overall continuity score, we aggregate the local metrics over a range of small lags, for example, by taking $\mu = \sum_{t=1}^T w_t \mu_t$, where the weights $w_t$ decrease with $t$, and $T$ sets the maximum lag under consideration.

We do not wish to claim that this is the best metric for measuring dataset continuity, and we hope future work will develop more alternatives and compare them systematically. Still, the metric $\mu_t$ proposed here has several appealing properties: it does not depend explicitly on the sequence length $L$; its dependence on the ambient dimension $d$ is entirely absorbed into the choice of kernel $K$, so once $K$ is specified, the metric applies seamlessly across different representation spaces; the correction term $\beta$ removes the background similarity between unrelated or far-apart elements, so that $\mu_t$ reflects local continuity rather than the global geometry of the representation space (e.g., see~\citet{yu2025understanding2}); and unlike a pointwise quantity such as a Lipschitz constant, $\mu_t$ provides a more stable measurement by aggregating the continuity structure across the entire sequences.

By computing $\mu$ for the six Long-Range Arena (LRA) tasks, we indeed find that tasks that are intuitively more continuous, such as flattened images, receive much higher continuity scores than intuitively discrete tasks such as text\footnote{To compute $\mu$, one first needs to embed the raw data into a hidden space. Here, we use the encoders of pretrained S4 models; using the encoders of S6 models or Transformers leads to very similar scores.}; moreover, $\mu$ indeed positively correlates to $\eta$ in our synthetic dataset (see~\Cref{fig:eta_mu}). Using this metric, in~\Cref{fig:continuoustasks} (Right), we observe that S4 gradually outperforms S6 and Transformers as the tasks become more continuous. This provides further evidence that models with stronger continuity are better suited to the more continuous tasks.

\section{Beyond inductive biases: What does continuity buy us?}\label{sec:accelerate}

\begin{wrapfigure}{r}{0.54\textwidth}
    \vspace{-1em}
    \begin{minipage}{0.52\textwidth}
    \captionsetup{type=algorithm}
    \begin{mdframed}[linewidth=0.8pt,roundcorner=3pt,innermargin=6pt]
    \small
    \begin{algorithmic}[1]
        \State \textbf{Input:} Training set $\mathcal{D}=\{(\mathbf{u}^{(i)},{y}^{(i)})\}$, stage strides $r_1 > \cdots > r_S = 1$, epochs $E_1,\ldots,E_S$
        \State \textbf{Output:} Trained model parameters
        \State Initialize S4 parameters $(\mathbf{A},\mathbf{B},\mathbf{C},\mathbf{D},\Delta)$
        \State $r_0 \gets r_1$
        \For{$s=1,\ldots,S$}
            \State $\Delta \gets \Delta \cdot (r_{s} / r_{s-1})$ \Comment{See~\cref{eq.discretizeLTItau}}
            \ForAll{$(\mathbf{u}^{(i)},{y}^{(i)}) \in \mathcal{D}$}
                \State Form subsampled input $\mathbf{u}^{(i,s)} = \bigl(\mathbf{u}^{(i)}_0,\mathbf{u}^{(i)}_{r_s},\mathbf{u}^{(i)}_{2r_s},\ldots\bigr)$\footnote{Average pooling can be used instead for discrete tasks.}
            \EndFor
            \State Train the model on dataset $\mathcal{D}^{(s)}=\{(\mathbf{u}^{(i,s)},{y}^{(i)})\}$ for $E_s$ epochs.
        \EndFor
    \end{algorithmic}
    \end{mdframed}
    \captionof{algorithm}{Stage-wise training via temporal subsampling. At each stage, the input is subsampled by stride $r_s$ and $\Delta$ is scaled by the same factor.}
    \label{alg:subsample}
    \end{minipage}
\end{wrapfigure}

At this point, we have identified continuity as an important inductive bias in sequential modeling. This has several consequences. For instance, it can help us anticipate which models are more likely to succeed on continuous or discrete tasks without fully training them, and it suggests how one might tune SSMs toward problems with different levels of continuity. But is continuity only useful as an inductive bias? In fact, continuity has broader algorithmic implications. We conclude the paper with one such application: accelerating training.

Training is often the most computationally expensive part of using sequential models. Existing acceleration methods, often inspired by reduced-order modeling, typically focus on compressing the state dimension, i.e., the hidden state $\mathbf{x}$ in~\cref{eq.contdyn}~\citep{chahine2025curious}. Such approaches can reduce the cost of later training epochs. Here, we show that continuity in S4 suggests a complementary direction: as observed in~\Cref{fig:DSresample}, S4 remains stable under temporal refinement. This motivates a natural speculation: if an S4 model can adapt well to higher-resolution data, why not begin training at a lower resolution?

This leads to a stage-wise training strategy. We start by subsampling the input sequence to a coarse temporal resolution, train the model on this cheaper version of the task, and then progressively refine the resolution while continuing to train the same model (see~\Cref{alg:subsample}). One may also view each stage as fine-tuning the model on a more resolved version of the data.

\begin{figure}[!htb]
    \centering
    \begin{overpic}[width=1\textwidth]{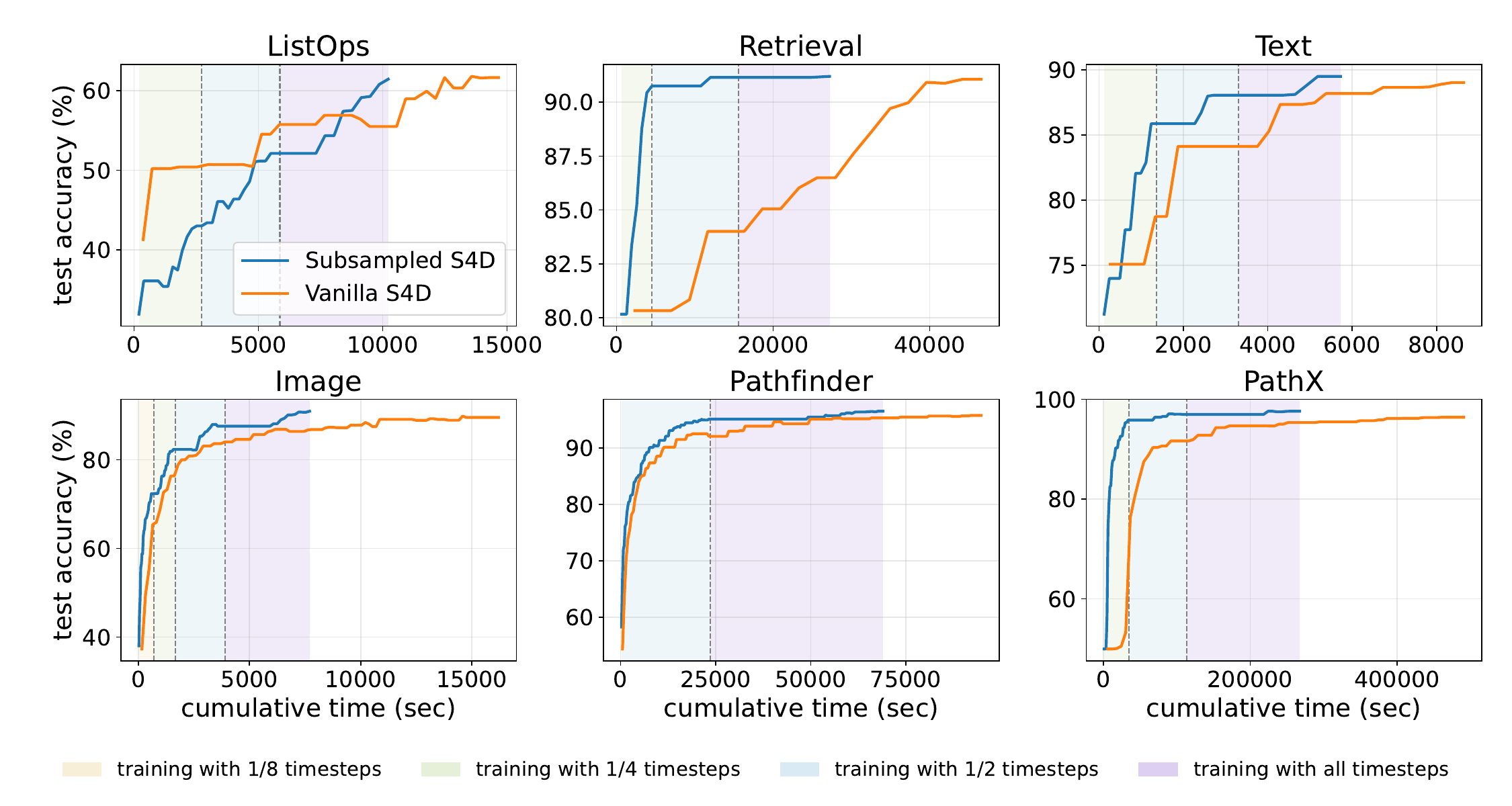}
            
    \end{overpic}
    \caption{Test accuracy plotted against cumulative training time on the six LRA tasks. We compare standard training with our stage-wise subsampling strategy, which starts from low-resolution inputs and progressively refines the temporal resolution during training. On five of the six tasks, the time--accuracy curve of our method stays above that of standard training, indicating that it reaches essentially any target accuracy in less wall-clock time. On several tasks, especially the more continuous ones, it also attains a higher final accuracy. All experiments are done on a single NVIDIA L40S GPU.}
    \label{fig:subsampledtraining}
\end{figure}

\begin{table}[!htb]
\centering
\caption{Relative time needed to reach target test accuracies on Long-Range Arena. Each numeric entry reports the ratio between the time needed by the subsampled S4D model and the time needed by the vanilla S4D model. Values below $1$ indicate faster convergence for the subsampled model.}
\label{tab:lra_time_accuracy_ratios}

\begin{subtable}{0.31\textwidth}
\centering
\caption{ListOps}
\scriptsize
\setlength{\tabcolsep}{3.2pt}
\renewcommand{\arraystretch}{1.05}
\begin{tabular}{c|ccccc}
\toprule
Acc. & 42 & 44 & 46 & 48 & 50 \\
\cmidrule(lr){1-6}
Ratio & \cellcolor[HTML]{F8696B}5.60 & \cellcolor[HTML]{F8696B}6.49 & \cellcolor[HTML]{F8696B}7.15 & \cellcolor[HTML]{F8696B}7.23 & \cellcolor[HTML]{F8696B}6.67 \\
\midrule
Acc. & 52 & 54 & 56 & 58 & 60 \\
\cmidrule(lr){1-6}
Ratio & \cellcolor[HTML]{FED981}1.07 & \cellcolor[HTML]{F8696B}1.61 & \cellcolor[HTML]{FDCA7E}1.13 & \cellcolor[HTML]{E4E382}0.83 & \cellcolor[HTML]{DFE282}0.80 \\
\bottomrule
\end{tabular}
\end{subtable}
\begin{subtable}{0.31\textwidth}
\centering
\caption{Retrieval}
\scriptsize
\setlength{\tabcolsep}{3.2pt}
\renewcommand{\arraystretch}{1.05}
\begin{tabular}{c|ccccc}
\toprule
Acc. & 81 & 82 & 83 & 84 & 85 \\
\cmidrule(lr){1-6}
Ratio & \cellcolor[HTML]{7BC57C}0.15 & \cellcolor[HTML]{7CC57C}0.16 & \cellcolor[HTML]{7EC67D}0.17 & \cellcolor[HTML]{80C67D}0.19 & \cellcolor[HTML]{78C47C}0.14 \\
\midrule
Acc. & 86 & 87 & 88 & 89 & 90 \\
\cmidrule(lr){1-6}
Ratio & \cellcolor[HTML]{75C37C}0.12 & \cellcolor[HTML]{73C37C}0.10 & \cellcolor[HTML]{72C27C}0.10 & \cellcolor[HTML]{72C27C}0.10 & \cellcolor[HTML]{72C27C}0.10 \\
\bottomrule
\end{tabular}
\end{subtable}
\begin{subtable}{0.31\textwidth}
\centering
\caption{Text}
\scriptsize
\setlength{\tabcolsep}{3.2pt}
\renewcommand{\arraystretch}{1.05}
\begin{tabular}{c|ccccc}
\toprule
Acc. & 77 & 78 & 79 & 80 & 81 \\
\cmidrule(lr){1-6}
Ratio & \cellcolor[HTML]{B0D47F}0.49 & \cellcolor[HTML]{BED880}0.59 & \cellcolor[HTML]{AED47F}0.48 & \cellcolor[HTML]{AFD47F}0.49 & \cellcolor[HTML]{AFD47F}0.49 \\
\midrule
Acc. & 82 & 83 & 84 & 85 & 86 \\
\cmidrule(lr){1-6}
Ratio & \cellcolor[HTML]{AFD47F}0.49 & \cellcolor[HTML]{C3DA81}0.62 & \cellcolor[HTML]{C4DA81}0.62 & \cellcolor[HTML]{92CC7E}0.30 & \cellcolor[HTML]{BFD980}0.59 \\
\bottomrule
\end{tabular}
\end{subtable}

\vspace{0.8em}

\begin{subtable}{0.31\textwidth}
\centering
\caption{Image}
\scriptsize
\setlength{\tabcolsep}{3.2pt}
\renewcommand{\arraystretch}{1.05}
\begin{tabular}{c|ccccc}
\toprule
Acc. & 40 & 45 & 50 & 55 & 60 \\
\cmidrule(lr){1-6}
Ratio & \cellcolor[HTML]{7FC67D}0.18 & \cellcolor[HTML]{83C77D}0.21 & \cellcolor[HTML]{8CCA7D}0.26 & \cellcolor[HTML]{86C87D}0.23 & \cellcolor[HTML]{9BCE7E}0.36 \\
\midrule
Acc. & 65 & 70 & 75 & 80 & 85 \\
\cmidrule(lr){1-6}
Ratio & \cellcolor[HTML]{A9D27F}0.45 & \cellcolor[HTML]{ACD37F}0.47 & \cellcolor[HTML]{DAE082}0.76 & \cellcolor[HTML]{D1DE81}0.71 & \cellcolor[HTML]{BCD880}0.57 \\
\bottomrule
\end{tabular}
\end{subtable}
\begin{subtable}{0.31\textwidth}
\centering
\caption{Pathfinder}
\scriptsize
\setlength{\tabcolsep}{3.2pt}
\renewcommand{\arraystretch}{1.05}
\begin{tabular}{c|ccccc}
\toprule
Acc. & 59 & 63 & 67 & 71 & 75 \\
\cmidrule(lr){1-6}
Ratio & \cellcolor[HTML]{9ACE7E}0.35 & \cellcolor[HTML]{A1D07F}0.40 & \cellcolor[HTML]{A2D07F}0.40 & \cellcolor[HTML]{A6D17F}0.43 & \cellcolor[HTML]{B0D47F}0.49 \\
\midrule
Acc. & 79 & 83 & 87 & 91 & 95 \\
\cmidrule(lr){1-6}
Ratio & \cellcolor[HTML]{B4D580}0.52 & \cellcolor[HTML]{D9E082}0.76 & \cellcolor[HTML]{D5DF82}0.73 & \cellcolor[HTML]{CFDD81}0.69 & \cellcolor[HTML]{A6D17F}0.43 \\
\bottomrule
\end{tabular}
\end{subtable}
\begin{subtable}{0.31\textwidth}
\centering
\caption{PathX}
\scriptsize
\setlength{\tabcolsep}{3.2pt}
\renewcommand{\arraystretch}{1.05}
\begin{tabular}{c|ccccc}
\toprule
Acc. & 51 & 56 & 61 & 66 & 71 \\
\cmidrule(lr){1-6}
Ratio & \cellcolor[HTML]{7CC57C}0.16 & \cellcolor[HTML]{7CC57C}0.16 & \cellcolor[HTML]{7DC67D}0.17 & \cellcolor[HTML]{7EC67D}0.17 & \cellcolor[HTML]{7FC67D}0.18 \\
\midrule
Acc. & 76 & 81 & 86 & 91 & 96 \\
\cmidrule(lr){1-6}
Ratio & \cellcolor[HTML]{81C77D}0.19 & \cellcolor[HTML]{81C77D}0.20 & \cellcolor[HTML]{85C87D}0.22 & \cellcolor[HTML]{87C87D}0.23 & \cellcolor[HTML]{80C67D}0.19 \\
\bottomrule
\end{tabular}
\end{subtable}

\begin{center}
\begin{tikzpicture}
\begin{axis}[
    hide axis,
    scale only axis,
    height=0pt,
    width=0pt,
    colormap={greenyellowred}{%
        rgb255(0cm)=(99,190,123);
        rgb255(1cm)=(255,235,132);
        rgb255(1.5cm)=(248,105,107)
    },
    colorbar horizontal,
    point meta min=0,
    point meta max=1.5,
    colorbar style={
        width=0.55\textwidth,
        height=0.12cm,
        xtick={0,1,1.5},
        xticklabels={0,1,$\geq 1.5$},
        xlabel={subsampled time / vanilla time},
        xlabel style={font=\small},
        tick label style={font=\small},
    },
]
\addplot [draw=none] coordinates {(0,0) [0]};
\end{axis}
\end{tikzpicture}
\end{center}

\end{table}

Conceptually, our method operates entirely on the time axis, so it can be combined with state-pruning techniques for additional speedups. Compared to those methods, it has two notable differences:
\begin{enumerate}[leftmargin=*]
    \item Our method lowers the cost of the early training stages by initially using subsampled sequences, so computation is cheap before the model is exposed to the full-resolution data. As a result, it can reach a given accuracy substantially earlier in wall-clock time.
    \item Our method does not modify or compress the model itself, so it preserves the full expressive power of the original architecture. The tradeoff is that it does not reduce inference cost.
\end{enumerate}

Our method is also connected to the multi-stage framework of~\citet{chang2017multi}, which progressively increases network depth during training. The two approaches nonetheless rely on different notions of continuity:~\citet{chang2017multi} exploits continuity in model depth, whereas we exploit continuity in time for continuous sequential models. In~\Cref{fig:subsampledtraining} and~\Cref{tab:lra_time_accuracy_ratios}, we show that, for five of the six LRA tasks, the time--accuracy curve of our method consistently stays above that of standard training (see Appendix~\ref{app:sec5exp} for more details). In other words, our approach reaches \textit{any} attainable accuracy in less time. Moreover, on several tasks --- especially the more continuous ones, such as Pathfinder, PathX, and Image --- it also achieves a higher final accuracy than standard training. One possible explanation is that, for tasks with stronger continuous structure, starting from low-resolution data encourages the model to learn coarse patterns first before fitting finer-scale details at later stages. This may act as a form of implicit regularization. We leave a more precise understanding of this phenomenon to future work. In addition, our subsampling method echoes the significance of model continuity; in particular, it does not help the more discrete-in-practice models such as S6 (see~\Cref{fig:mambacifar}).

\section{Conclusion}

In this paper, we propose continuity as a principle for understanding when sequential models match the temporal structure of their tasks. Our results show that models differ in how their predictions behave under temporal refinement, tasks differ in how much continuous temporal structure their sequences contain, and performance improves when these two forms of continuity are aligned. This perspective separates continuous-time motivation from continuous behavior: although S4, S5, and S6 all arise from state-space formulations, they have different refinement behaviors once implemented as discrete models. More broadly, understanding task continuity together with the continuous inductive bias of an architecture can help predict which models should be favored in which temporal regimes. This suggests that continuity is a useful law-like organizing principle for sequential modeling. Future work should develop a broader theory of how continuity interacts with other inductive biases.

\subsubsection*{Acknowledgments}

We would like to acknowledge support from the U.S. Department of Energy, Office of Science, Office of Advanced Scientific Computing Research, EXPRESS: 2025 Exploratory Research for Extreme-Scale Science program, under Contract Number DE-AC02-05CH11231 at Berkeley Lab. The views, opinions, and/or findings expressed are those of the authors and should not be interpreted as representing the official views or policies of the Department of Defense or the U.S. Government.

\subsubsection*{Reproducibility Statement}

Code scripts that are used to produce the numerical and empirical results presented in this paper can be found at \url{https://github.com/AnnanYu/ContinuousLaws}.

\bibliographystyle{plainnat}
\bibliography{references}

\clearpage

%%%%%%%%%%%%%%%%%%%%%%%%%%%%%%%%%%%%%%%%%%%%%%%%%%%%%%%%%%%%

\appendix

\section{Related Work}\label{app:relatedwork}

\paragraph{Continuous Models.} Continuous-time neural models arose from the observation that many sequential and dynamical processes are more naturally described by latent states evolving on a continuous time axis than by fixed-step recursions. A major modern entry point was Neural ODEs~\citep{chen2018neural,kim2021stiff,poli2019graph}, which replaced a finite stack of residual layers with a learned vector field integrated by an ODE solver, thereby turning depth itself into a continuous variable. This perspective quickly led to continuous-time latent-variable models for irregularly sampled data, such as Latent ODEs and ODE-RNNs~\citep{rubanova2019latent}, where hidden states evolve continuously between observations and are updated when new data arrive. A related line of work emphasizes \emph{adaptive} or \emph{gated} continuous dynamics: GRU-ODE-Bayes imports GRU-style gating into a continuous-time evolution rule~\citep{de2019gru}, while Liquid Time-Constant networks and their closed-form variants modulate effective time constants and input responses through state- and input-dependent interactions~\citep{hasani2021liquid,hasani2022closed}. From this viewpoint, gating can be seen as a mechanism for making the underlying dynamics selective or input-adaptive, rather than purely time-invariant. Our work is conceptually adjacent to this literature, but asks a different question: when a model is presented as continuous-time, to what extent does its discrete behavior actually preserve that continuity in practice? A closely related perspective is developed in~\citet{queiruga2020continuous}, which argues that the connection between deep networks and continuous dynamics should be understood through the lens of numerical integration more carefully, rather than through a purely formal Euler-style analogy, and emphasizes invariance across different discrete realizations of the same underlying dynamics.

\paragraph{Continuity of State-space Models.} State-space models form another major class of continuous models in modern sequence learning: the original linear state-space layer and S4 explicitly start from a continuous-time system and then discretize it for efficient computation~\citep{gu2021efficiently}, while S4D~\citep{gu2021efficiently}, S5~\citep{smith2022simplified}, and S6~\citep{gu2023mamba} simplify, generalize, or make selective this same underlying state-space viewpoint. In particular, S4 builds an efficient structured parameterization of a continuous-time linear system, S4D shows that a diagonal realization can retain much of this behavior when properly initialized, S5 recasts the layer as a multi-input multi-output SSM with efficient scans, and S6/Mamba makes the discretized dynamics input-dependent through selectivity.

This continuous perspective is not limited to the canonical S4/S4D/S5 family. Classical systems theory gives a tight correspondence between state-space realizations, transfer functions, and convolutional descriptions~\citep{antoulas2005approximation}, and this viewpoint has become increasingly explicit in machine learning:~\citet{gu2021combining} frames sequence layers as recurrent, convolutional, and continuous-time models simultaneously;~\citet{yu2024hope} develops a different continuous-time parameterization through Hankel operators; and even seemingly more ``discrete'' sequence models such as linear recurrent units~\citep{orvieto2023resurrecting} and smooth long convolutions~\citep{fu2023simple} can often be interpreted through equivalent linear dynamical or transfer-function descriptions~\citep{parnichkun2024state,yu2023robustifying}.

A growing line of work leverages this continuity directly rather than treating it as background intuition. For example,~\citet{gu2022train} gives a more explicit continuous interpretation of HiPPO-style initialization;~\citet{liu2024autocorrelation} studies initialization through temporal structure; and~\citet{yu2024tuning} analyzes how continuous-time transfer functions induce frequency preferences and how those preferences can be modified. On the theoretical side,~\citet{wang2023stablessm} studies long-term memory in SSMs from the viewpoint of parameterization, showing that stable reparameterizations can preserve expressive long-range dynamics while improving optimization stability, and~\citet{muca2024theoretical} proposes a continuous-time-inspired framework for understanding selective SSMs via controlled dynamical systems and rough-path tools.

This continuity has also been used in applications, especially when the sampling rate may vary across training and deployment. In event-based vision, S4/S5-style models have been used precisely because their learnable time-scale parameter allows inference at frequencies different from those seen during training~\citep{zubic2024state}; in time-series forecasting,~\citet{graf2025flowstate} uses an S5-based encoder together with continuous-time scaling to target sampling-rate invariance; and in scientific machine learning, deep SSMs have recently been shown to outperform Transformers on rainfall--runoff simulation~\citep{wang2025deep}. These works are complementary to ours: they exploit the continuous formulation of SSMs in practice, whereas our focus is to ask when that continuous interpretation is actually faithful at the level of the induced discrete model.

\paragraph{Inductive Biases of Sequential Models.} Inductive bias is a foundational concept in machine learning: in the classical view, some bias is necessary for any learner to generalize beyond the training data, and modern deep learning can often be understood as the interaction between architectural, optimization, and data-induced biases~\citep{mitchell1980need}. In general neural networks, one of the most studied examples is {frequency bias} or {spectral bias}, namely the tendency to learn lower-frequency components earlier or more easily than higher-frequency ones~\citep{rahaman2019spectral,yu2022tuning}.

For sequential models, inductive biases are especially central because sequence architectures differ sharply in how they encode time, locality, memory, and scale. A large body of work studies such biases either explicitly or implicitly~\citep{yu2025understanding,ebrahimi2026induction}. Among the most common examples are \emph{recency bias}, where recent tokens or states receive disproportionate weight~\citep{hegazyrecency}; \emph{frequency bias}, where models prefer smoother or lower-frequency temporal structure~\citep{yu2024tuning}; \emph{locality bias}, where nearby tokens interact more strongly than distant ones~\citep{beltagy2020longformer}; \emph{multi-scale or hierarchical bias}, where architectures favor representations that decompose across resolutions~\citep{chung2016hierarchical}; \emph{position or order bias}, induced by causal masking or positional encodings~\citep{wu2025emergence}; and \emph{state-tracking or finite-memory bias}, which governs how information is preserved across long horizons~\citep{wang2023stablessm}. These biases appear in different forms across convolutions, recurrent models, Transformers, and SSMs, and they often interact rather than acting in isolation.

\paragraph{Compression of State-Space Models.}
The idea of compressing an SSM naturally connects to the classical literature on model order reduction (MOR), which seeks lower-order dynamical systems that preserve the input--output behavior of a higher-order model~\citep{antoulas2005approximation}. Canonical tools in this area include balanced truncation~\citep{gugercin2004survey} and Hankel-norm approximation~\citep{glover1984all,yu2024leveraging}, which come with strong approximation guarantees. A recurring tension, however, is that many of these methods rely on similarity transforms or dense reduced realizations that can destroy the diagonal or structured parameterizations exploited by modern SSMs for efficiency and stability. This challenge is already emphasized in recent SSM pruning work, which positions post-training compression as a form of MOR adapted to deep stacks and structured state spaces.

Recent work has therefore focused on {structure-aware} compression schemes for learned SSMs. LAST introduces layer-adaptive state pruning based on per-state $\mathcal{H}_\infty$ scores and a global cross-layer ranking, showing that trained SSMs often contain substantial removable redundancy~\citep{gwak2024layer}. AIRE-Prune replaces this worst-case perspective with an asymptotic impulse-response energy criterion, again in a post-training setting, and reports strong compression with little accuracy loss~\citep{padhy2026aire}.

A complementary direction aims to compress SSMs {during} optimization rather than after training. For example,~\citet{chahine2025curious} uses Hankel-singular-value ideas to begin with a larger linear SSM and then reduce its dimension over the course of training, improving the tradeoff between expressivity and efficiency. Other orthogonal directions reduce cost by changing the inference mechanism rather than pruning states directly, for example, via transfer-function or state-free inference formulations~\citep{parnichkun2024state}. Our approach is orthogonal to these lines of work: instead of reducing the state dimension, we compress the {time axis} during training, and can therefore be combined with state-pruning or in-training reduction methods.

\clearpage

\section{Proof of Theoretical Results}\label{app:proofs}

In this appendix, we present the proofs of our theoretical claims that S4 is a more continuous class of models than S6. To begin with, we prove~\Cref{lem.discretize} about a general continuous system, and the rest of the claims follow easily.

To prove~\Cref{lem.discretize}, we leverage a key property of the ZOH discretization: if the function $u(t)$ is constant on each interval between the sampling points, then the ZOH discretization does not induce any error. In our assumption, however, $u(t)$ is obviously not piecewise constant, but we can approximate $u(t)$ by a piecewise-constant function $v(t)$. Now, by a triangle inequality, the difference between the continuous output on $u(t)$ and the discrete output on $u(\tau k)$ is upper-bounded by the sum of three terms: 
\begin{itemize}[leftmargin=*]
    \item the difference between the continuous output on $u(t)$ and the continuous output on $v(t)$,
    \item that between the continuous output on $v(t)$ and the discrete output on $v(\tau k)$,
    \item and that between the discrete output on $u(\tau k)$ and the discrete output on $v(\tau k)$.
\end{itemize}
By defining $v$ so that $u(\tau k) = v(\tau k)$, the last term always vanishes. Moreover, for the ZOH discretization, the second term is also zero. It therefore remains to control the first term, which reduces to controlling the growth of a system's states and output given a zero initial condition and a small input $u(t) - v(t)$. Hence, it naturally fits into a Gr\"onwall-type analysis. We now formalize this argument.

\begin{proof}[Proof of~\Cref{lem.discretize}]
    Fix a $\tau \in (0,1)$, define a step function $v^\tau: [0,1] \rightarrow \R$ by
\[
	v^\tau(t) = u(\tau \lfloor t / \tau \rfloor), \qquad t \in [0,1].
\]
Since $u(t)$ is $L_u$-Lipschitz continuous, we have
\begin{equation}\label{eq.uvdist}
	|u(t) - v^\tau(t)| \leq L_u \tau, \qquad 0 \leq t \leq 1.
\end{equation}
Define $\boldsymbol{\xi}^\tau(t)$ and $z^\tau(t)$ to be the solution of the continuous system on input $v^\tau$, i.e.,
\begin{equation}\label{eq.stepLTI}
    	\begin{aligned}
		(\boldsymbol{\xi}^\tau)'(t) &= \Delta(v^\tau(t)) (\mathbf{A} \boldsymbol{\xi}^\tau(t) + \mathbf{B}(v^\tau(t)) v^\tau(t)), \\
		z^\tau(t) &= \mathbf{C}(v^\tau(t))^\top \boldsymbol{\xi}^\tau(t) + D v^\tau(t),
	\end{aligned} \qquad\qquad \boldsymbol{\xi}^\tau(0) = \boldsymbol{0}.
\end{equation}
In what follows, we break the proof into two parts, corresponding to the case where ZOH or bilinear discretization is used, respectively.

\paragraph{CASE I: ZOH Discretization.}
Define $\mathbf{z}^\tau = (z_0^\tau, \ldots, z_{\lfloor 1 / \tau \rfloor}^\tau)$ be the solution of the discrete system
\[
	\begin{aligned}
		\boldsymbol{\xi}^\tau_{k+1} &= \overline{\mathbf{A}}^\tau_k \boldsymbol{\xi}^\tau_{k} + \overline{\mathbf{B}}^\tau_k v(\tau k), \\
		z_k^\tau &= (\overline{\mathbf{C}}^\tau_k)^\top \boldsymbol{\xi}_{k}^\tau + \overline{D} v(\tau k),
	\end{aligned} \qquad\qquad \boldsymbol{\xi}^\tau_0 = \boldsymbol{0},
\]
where $\overline{\mathbf{A}}^\tau_k$, $\overline{\mathbf{B}}^\tau_k$, $\overline{\mathbf{C}}^\tau_k$, and $\overline{{D}}$ are defined by the ZOH formula in~\cref{eq.discretizeLTI}. Note that since we have $v(\tau k) = u(\tau k)$ for every $k = 0, \ldots, \lfloor 1 / \tau \rfloor$, we have
\begin{equation}\label{eq.yandz}
	z_k^\tau = y_k^\tau, \qquad 1 \leq k \leq \lfloor 1 / \tau \rfloor.
\end{equation}
Moreover, the discrete system defined by the ZOH formula in~\cref{eq.discretizeLTI} also forms a ZOH discretization of the system in~\cref{eq.stepLTI}. Since $v^\tau$ is a step function with a step size of $\tau$, by definition, we have
\begin{equation}\label{eq.ZOHv}
	z_k^\tau = z^\tau(\tau k), \qquad 1 \leq k \leq \lfloor 1 / \tau \rfloor.
\end{equation}
Next, define
\[
	\boldsymbol{\varepsilon}^\tau = \boldsymbol{\xi}^\tau - \mathbf{x} : [0,1] \rightarrow \R.
\]
We have
\begin{equation*}
	\begin{aligned}
	(\boldsymbol{\varepsilon}^\tau)'(t) &= \Delta(v^\tau(t)) (\mathbf{A} \boldsymbol{\xi}^\tau(t) + \mathbf{B}(v^\tau(t)) v^\tau(t)) - \Delta(u(t)) (\mathbf{A} \mathbf{x}(t) + \mathbf{B}(u(t)) u(t)) \\
	&= \Delta(v^\tau(t)) (\mathbf{A} \boldsymbol{\xi}^\tau(t) + \mathbf{B}(v^\tau(t)) v^\tau(t)) \\
        &\qquad - (\Delta(v^\tau(t)) + (\Delta(u(t)) - \Delta(v^\tau(t)))) (\mathbf{A} \mathbf{x}(t) + \mathbf{B}(u(t)) u(t)) \\
	&= \underbrace{\Delta(v^\tau(t)) \mathbf{A} \boldsymbol{\varepsilon}^\tau(t)}_{\boldsymbol{\varepsilon}^\tau_1(t)} + \underbrace{\Delta(v^\tau(t)) (\mathbf{B}(v^\tau(t)) v^\tau(t) - \mathbf{B}(u(t)) u(t))}_{\boldsymbol{\varepsilon}^\tau_2(t)} \\
	&\qquad + \underbrace{(\Delta(u(t)) - \Delta(v^\tau(t))) \mathbf{A}\mathbf{x}(t)}_{\boldsymbol{\varepsilon}^\tau_3(t)} + \underbrace{(\Delta(u(t)) - \Delta(v^\tau(t))) \mathbf{B}(u(t)) u(t)}_{\boldsymbol{\varepsilon}^\tau_4(t)}
	\end{aligned}
\end{equation*}
We then control these four terms separately. First, we have
\begin{equation}\label{eq.E1}
	\|\boldsymbol{\varepsilon}^\tau_1(t)\|_2 \leq |\Delta(v^\tau(t))| \|\mathbf{A}\|_2 \|\boldsymbol{\varepsilon}^\tau(t)\|_2 \leq M_\Delta \|\mathbf{A}\|_2 \|\boldsymbol{\varepsilon}^\tau(t)\|_2, \qquad 0 \leq t \leq 1,
\end{equation}
where we control $|\Delta(v^\tau(t))|$ using the fact that $\text{Range}(v) \subset \text{Range}(u)$. Next, we have
\begin{equation}\label{eq.E2}
\begin{aligned}
	\|\boldsymbol{\varepsilon}^\tau_2(t)\|_2 &\leq M_\Delta (\|\mathbf{B}(v^\tau(t)) - \mathbf{B}(u(t))\|_2 |u(t)| + \|\mathbf{B}(u^\tau(t))\|_2 |v^\tau(t) - u(t)|), \\
	&= M_\Delta(L_B L_u M_u \tau + M_B L_u \tau), \\
	&= M_\Delta (L_B M_u + M_B) L_u \tau,  \qquad 0 \leq t \leq 1.
\end{aligned}
\end{equation}
From
\[
\mathbf x'(t)=\Delta(u(t))(\mathbf A\mathbf x(t)+\mathbf B(u(t))u(t)),\qquad \mathbf x(0)=0,
\]
we obtain
\[
\|\mathbf x'(t)\|_2
\le
M_\Delta\|\mathbf A\|_2\|\mathbf x(t)\|_2 + M_\Delta M_B M_u.
\]
Hence, by Gr\"onwall's inequality,
\[
\|\mathbf x(t)\|_2
\le
\frac{M_B M_u}{\|\mathbf A\|_2}\bigl(e^{M_\Delta\|\mathbf A\|_2}-1\bigr),
\qquad 0\le t\le 1.
\]
Therefore, the third can be bounded by
\begin{equation}\label{eq.E3}
\begin{aligned}
	\|\boldsymbol{\varepsilon}^\tau_3(t)\|_2 &\leq L_\Delta L_u \tau \|\mathbf{A}\|_2 \|\mathbf{x}(t)\|_2 \leq L_\Delta L_u \tau M_B M_u \left(e^{M_\Delta \|\mathbf{A}\|_2} - 1\right),  \qquad 0 \leq t \leq 1.
\end{aligned}
\end{equation}
The final term satisfies that
\begin{equation}\label{eq.E4}
\begin{aligned}
	\|\boldsymbol{\varepsilon}^\tau_4(t)\|_2 &\leq L_\Delta L_u \tau M_B M_u,  \qquad 0 \leq t \leq 1.
\end{aligned}
\end{equation}
Now, combining~\cref{eq.E1}-\cref{eq.E4}, for every $0 \leq t \leq 1$, we have
\begin{equation*}
	\|(\boldsymbol{\varepsilon}^\tau)'(t)\|_2 \leq M_\Delta \|\mathbf{A}\|_2 \|\boldsymbol{\varepsilon}^\tau(t)\|_2 + \left[M_\Delta (L_B M_u + M_B) L_u + L_\Delta L_u M_B M_u e^{M_\Delta \|A\|_2}\right]\tau.
\end{equation*}
Since $\boldsymbol{\varepsilon}^\tau(0) = \boldsymbol{0}$, by Gr\"onwall's inequality again, we have
\[
	\|\boldsymbol{\varepsilon}^\tau(t)\|_2 \leq \frac{M_\Delta (L_B M_u + M_B) L_u + L_\Delta L_u M_B M_u e^{M_\Delta \|\mathbf{A}\|_2}}{M_\Delta \|\mathbf{A}\|_2} \; (\text{exp}(M_\Delta \|\mathbf{A}\|_2 t) - 1) \tau.
\]
Hence, for every $0 \leq k \leq \lfloor 1 / \tau \rfloor$, by H\"older's inequality, we have
\begin{equation}\label{eq.contyandz}
\begin{aligned}
	|z^\tau(\tau k) - y(\tau k)| &\leq \|\mathbf{C}(u(k\tau))\|_2 \|\boldsymbol{\varepsilon}^\tau(k\tau)\|_2 \\
	&\!\!\!\!\!\leq M_C \frac{M_\Delta (L_B M_u + M_B) L_u + L_\Delta L_u M_B M_u e^{M_\Delta \|\mathbf{A}\|_2}}{M_\Delta \|\mathbf{A}\|_2} \; (\text{exp}(M_\Delta \|\mathbf{A}\|_2) - 1) \tau.
\end{aligned}
\end{equation}
Combining~\cref{eq.yandz},~\cref{eq.ZOHv}, and~\cref{eq.contyandz}, for every $0 \leq k \leq \lfloor 1 / \tau \rfloor$, we have
\[
	\begin{aligned}
		|y^\tau_k - y(\tau k)| &\leq |y^\tau_k - z^\tau_k| + |z^\tau_k - z^\tau(\tau k)| + |z^\tau(\tau k)- y(\tau k)| \\
		&\!\!\!\!\!\!\!\!\!\!\!\!\!\leq 0 + 0 + M_C \frac{M_\Delta (L_B M_u + M_B) L_u + L_\Delta L_u M_B M_u e^{M_\Delta \|\mathbf{A}\|_2}}{M_\Delta \|\mathbf{A}\|_2} \; (\text{exp}(M_\Delta \|\mathbf{A}\|_2) - 1) \tau.
	\end{aligned}
\]
This proves the result for the ZOH discretization.

\paragraph{CASE II: Bilinear Discretization.} Define also the discrete trajectory $\mathbf{z}^\tau = (z_0^\tau,\dots,z_{\lfloor 1/\tau\rfloor}^\tau)$ by the bilinear recursion
\[
	\begin{aligned}
		\boldsymbol{\xi}_{k+1}^\tau &= \overline{\mathbf{A}}_k^\tau \boldsymbol{\xi}_{k}^\tau + \overline{\mathbf{B}}_k^\tau v^\tau(\tau k), \\
		z_k^\tau &= (\overline{\mathbf{C}}_k^\tau)^\top \boldsymbol{\xi}_{k}^\tau + \overline{D}\, v^\tau(\tau k),
	\end{aligned}
	\qquad\qquad \boldsymbol{\xi}_0^\tau = \boldsymbol{0},
\]
where $\overline{\mathbf{A}}_k^\tau,\overline{\mathbf{B}}_k^\tau,\overline{\mathbf{C}}_k^\tau,\overline{D}$ are now defined by the bilinear formula in~\cref{eq.discretizeLTI}. Since $v^\tau(\tau k)=u(\tau k)$ for every $k$, we again have
\begin{equation}\label{eq.yandz-bilinear}
	z_k^\tau = y_k^\tau, \qquad 0 \leq k \leq \lfloor 1/\tau \rfloor.
\end{equation}
Unlike the ZOH case, the bilinear discretization is not exact on the step input $v^\tau$, so in general $z_k^\tau \neq z(\tau k)$. We therefore decompose
\begin{equation}\label{eq.bilinear-split}
	|y_k^\tau - y(\tau k)|
	\leq
	|y_k^\tau - z^\tau_k|
	+
	|z^\tau_k - z(\tau k)|
	+
	|z(\tau k) - y(\tau k)|.
\end{equation}
The first term vanishes by~\cref{eq.yandz-bilinear}. The second term follows the same argument in the ZOH case and satisfies that
\begin{equation}\label{eq.bilinear-cont-part}
\begin{aligned}
	|\widehat{\mathbf{z}}(\tau k) - y(\tau k)|
	\leq
	M_C
	\frac{M_\Delta (L_B M_u + M_B)L_u + L_\Delta L_u M_B M_u e^{M_\Delta \|\mathbf{A}\|_2}}{M_\Delta \|\mathbf{A}\|_2}
	\left(e^{M_\Delta \|\mathbf{A}\|_2}-1\right)\tau.
\end{aligned}
\end{equation}
It remains to control $|z_k^\tau - z(\tau k)|$. Since $v^\tau$ is constant on each interval $[\tau k,\tau(k+1))$, the continuous dynamics on that interval reduce to the linear ODE
\[
	(\boldsymbol{\xi}^\tau)'(t)
	=
	\Delta(u(\tau k))
	\bigl(\mathbf{A}\boldsymbol{\xi}^\tau(t)+\mathbf{B}(u(\tau k))u(\tau k)\bigr),
\]
that is, to a constant-coefficient affine system
\[
	(\boldsymbol{\xi}^\tau)'(t)=\mathbf{M}_k \boldsymbol{\xi}^\tau(t)+\mathbf{g}_k,
\]
with
\[
	\mathbf{M}_k:=\Delta(u(\tau k))\mathbf{A},
	\qquad
	\mathbf{g}_k:=\Delta(u(\tau k))\mathbf{B}(u(\tau k))u(\tau k).
\]
The bilinear update is precisely the trapezoidal-rule discretization of this ODE. Since the exact solution of a constant-coefficient linear ODE is smooth, the trapezoidal rule has local truncation error of order $\mathcal O(\tau^3)$ on each step, uniformly in $k$. Moreover, the one-step maps are uniformly stable for sufficiently small $\tau$, i.e.,
\[
	\left\|\left(\mathbf I-\frac{\tau}{2}\mathbf M_k\right)^{-1}
	\left(\mathbf I+\frac{\tau}{2}\mathbf M_k\right)\right\|_2
	\le 1 + C\tau, \qquad 0 \leq k \leq \lfloor 1 / \tau \rfloor,
\]
for some constant $C>0$ independent of $k$.
Therefore, by the standard global error estimate for the trapezoidal rule, there exists a constant $C' > 0$, independent of $\tau$ and $k$, such that
\begin{equation}\label{eq.bilinear-disc-error}
	\|\boldsymbol{\xi}_k^\tau-\boldsymbol{\xi}^\tau(\tau k)\|_2
	\le
	C'\tau^2,
	\qquad
	0\le k\le \lfloor 1/\tau\rfloor.
\end{equation}
Consequently,
\begin{equation}\label{eq.bilinear-output-disc-error}
\begin{aligned}
	|z_k^\tau-z(\tau k)|
	&=
	|(\mathbf C(u(\tau k)))^\top(\boldsymbol{\xi}_k^\tau-\boldsymbol{\xi}^\tau(\tau k))| \leq
	M_C \|\boldsymbol{\xi}_k^\tau-\boldsymbol{\xi}^\tau(\tau k)\|_2 \leq
	M_C C' \tau^2.
\end{aligned}
\end{equation}
Combining~\cref{eq.bilinear-split,eq.bilinear-cont-part,eq.bilinear-output-disc-error}, for every $0 \leq k \leq \lfloor 1/\tau \rfloor$, we obtain
\[
\begin{aligned}
	|y_k^\tau \!-\! y(\tau k)| \leq
	M_C
	\frac{M_\Delta (L_B M_u + M_B)L_u \!+\! L_\Delta L_u M_B M_u e^{M_\Delta \|\mathbf{A}\|_2}}{M_\Delta \|\mathbf{A}\|_2}
	\left(e^{M_\Delta \|\mathbf{A}\|_2}\!-\!1\right)\tau \!+\! \mathcal{O}(\tau^2).
\end{aligned}
\]
This proves the desired estimate.
\end{proof}

Now that we have proved~\Cref{lem.discretize}, showing the corollaries on S4 and S6 is a simple task to plug in the maximum moduli and the Lipschitz constants into the lemma. This is particularly simple for S4 because the parameters are input-independent, making the Lipschitz constants $L_B = L_C = L_\Delta = 0$. For S6, the Lipschitz constants can be large, making the discretization errors usually larger.

\begin{proof}[Proof of~\Cref{cor:S4discrete}]
	Since $\mathbf{B}, \mathbf{C}$, and $\Delta$ are constant, we have $L_B = L_C = L_\Delta = 0$. Thus, plugging in the expressions gives us
	\begin{align*}
		&M_C \frac{M_\Delta (L_B M_u + M_B) L_u + L_\Delta L_u M_B M_u e^{M_\Delta \|\mathbf{A}\|_2}}{M_\Delta \|\mathbf{A}\|_2} \; (\text{exp}(M_\Delta \|\mathbf{A}\|_2) - 1) \\
		&\qquad= \frac{\|\mathbf{C}\|_2 \|\mathbf{B}\|_2 \Delta L_u (\text{exp}(\Delta \|\mathbf{A}\|_2) - 1)}{\Delta \|\mathbf{A}\|_2} = \frac{\|\mathbf{C}\|_2 \|\mathbf{B}\|_2 L_u (\text{exp}(\Delta \|\mathbf{A}\|_2) - 1)}{\|\mathbf{A}\|_2}.
	\end{align*}
	This proves the corollary.
\end{proof}

\begin{proof}[Proof of~\Cref{cor:S6discrete}]
    By definition, we have
\[
	L_B = \|\mathbf{W}_\mathbf{B}\|_2, \qquad L_C = \|\mathbf{W}_\mathbf{C}\|_2, \qquad L_\Delta = |w_\Delta|.
\]
Moreover, computing the maximum moduli of $\mathbf{B}, \mathbf{C}$, and $\Delta$ gives us
\[
	M_B = \|\mathbf{W}_\mathbf{B}\|_2 M_u, \qquad M_C = \|\mathbf{W}_\mathbf{C}\|_2 M_u, \qquad M_\Delta = \text{softplus}(|w_\Delta| M_u + b_\Delta).
\]
Plugging these expressions into~\Cref{lem.discretize}, we obtain
\begin{align*}
& M_C \frac{M_\Delta (L_B M_u + M_B) L_u + L_\Delta L_u M_B M_u e^{M_\Delta \|\mathbf{A}\|_2}}{M_\Delta \|\mathbf{A}\|_2}
\left(e^{M_\Delta \|\mathbf{A}\|_2} - 1\right) \\
&\le
\|\mathbf W_{\mathbf C}\|_2 M_u
\frac{
M_\Delta(\|\mathbf W_{\mathbf B}\|_2 M_u+\|\mathbf W_{\mathbf B}\|_2 M_u)L_u
+
|w_\Delta|L_u\,\|\mathbf W_{\mathbf B}\|_2 M_u^2 e^{M_\Delta\|\mathbf A\|_2}
}{
M_\Delta\|\mathbf A\|_2
}
\left(e^{M_\Delta\|\mathbf A\|_2}-1\right) \\
&=
\frac{
\|\mathbf W_{\mathbf C}\|_2 \|\mathbf W_{\mathbf B}\|_2 M_u^2 L_u
\left(2M_\Delta + |w_\Delta|M_u e^{M_\Delta\|\mathbf A\|_2}\right)
}{
M_\Delta\|\mathbf A\|_2
}
\left(e^{M_\Delta\|\mathbf A\|_2}-1\right).
\end{align*}
Finally, substituting
\[
M_\Delta=\text{softplus}(|w_\Delta|M_u+b_\Delta)
\]
yields the claimed estimate.
\end{proof}

\clearpage

\section{Details of Experimemts in~\Cref{sec:contmodel}}\label{app:sec3exp}

In this appendix, we present the details of the experiments in the main text that aim to understand the convergence of S4 and S6's discrete outputs to their underlying continuous outputs, respectively. There are two experiments: a numerical one and a synthetic training experiment on dynamical systems. We explain them separately.

\subsection{The numerical experiment}

In the experiment demonstrated in~\Cref{fig:numericaldiscretize}, we consider single-input, single-output dynamical systems. We randomly sample $20$ system--input pairs as follows. First, we generate the input by forming a random linear combination of the first $20$ Chebyshev polynomials on $[0,1]$:
\[
    u(t) = \sum_{j=1}^{20} C_j T_j(t), \qquad C_j \sim \mathcal{N}(0,1), \qquad 1 \leq j \leq 20,
\]
where $T_j$ is the $j$-th order Chebyshev polynomial basis function. To generate a random system, we let $\mathbf{A}$ be the diagonal HiPPO-LegS matrix used in many SSMs (see~\citep{gu2022parameterization}) and randomly sample
\[
    \mathbf{B} \sim \mathcal{N}(\boldsymbol{0}, \mathbf{I}_n), \qquad \mathbf{C} \sim \mathcal{N}(\boldsymbol{0}, \mathbf{I}_n),
\]
where $n = 8$ is the number of hidden states. Since the $D$ term does not induce a discretization error, we set it to zero. The matrices $\mathbf{B}$ and $\mathbf{C}$ are then used to form an S4 and an S6 system: for S4, they are directly used as the parameters of the LTI system; for S6, the matrices $\mathbf{B}$ and $\mathbf{C}$ are the parameters that define the projection maps $\mathbf{B}(\cdot)$ and $\mathbf{C}(\cdot)$. For S6, we additionally define the weights and biases in the gating term $\Delta(u)$:
\[
    w_\Delta \sim \mathcal{N}(0,1), \qquad b_\Delta = \text{softplus}^{-1}(0.01),
\]
and for S4, we define the constant sampling interval to be $\Delta = 0.01$, which is consistent with $b_\Delta$ used in S6. To obtain an accurate benchmark for the continuous-level output, we set $\tau_0 = 2^{-14}$ and use the fourth-order Runge--Kutta method to simulate the S4 and S6 systems. We then simulate the two systems using $\tau$ ranging from $2^{-10}$ to $2^{-2}$ and compute the $L_\infty$ errors. For a new scaling factor, we multiply it to the input function $u(t)$ and repeat the experiment. The experiments are repeated for $20$ different input $u(t)$ and S4/S6 systems. All numerical simulations are carried out using MATLAB R2025b.

\subsection{Experiment on dynamical systems}

In the next experiment that we present in~\Cref{fig:DSresample}, we consider four standard parameterized dynamical systems. In each case, the model receives a one-dimensional observed time series and predicts the underlying system parameters. We first explain these four classes of dynamical systems and then discuss the experimental setup.

\paragraph{Van der Pol oscillator.}
The Van der Pol oscillator is a classical nonlinear system originally introduced to model self-sustained oscillations in vacuum-tube circuits. For $\mu>0$, trajectories converge to a stable limit cycle; small $\mu$ yields near-sinusoidal oscillations, while large $\mu$ produces relaxation oscillations with pronounced slow--fast structure. The dynamics are
\begin{equation*}
    \ddot{x} - \mu(1-x^2)\dot{x} + x = 0,
\end{equation*}
or equivalently, in first-order form with state $(x,v)$,
\begin{equation*}
    \dot{x}=v, \qquad \dot{v}=\mu(1-x^2)v - x.
\end{equation*}
In our setting, the model observes $x(t)$ and predicts the scalar nonlinearity parameter $\mu$.

\paragraph{Damped harmonic oscillator.}
The damped harmonic oscillator is a standard linear model for oscillatory motion with energy loss, appearing in mechanical vibrations and electrical RLC circuits. Its behavior is governed by the natural frequency $\omega$ and damping ratio $\zeta$, which determine oscillation frequency and decay rate. The dynamics are
\begin{equation*}
    \ddot{x} + 2\zeta\omega\,\dot{x} + \omega^2 x = 0,
\end{equation*}
or in first-order form,
\begin{equation*}
    \dot{x}=v, \qquad \dot{v}=-2\zeta\omega\,v - \omega^2 x.
\end{equation*}
In our setting, the model observes $x(t)$ and predicts the parameters $(\omega,\zeta)$.

\paragraph{Ornstein--Uhlenbeck (OU) process.}
The OU process is a canonical continuous-time stochastic process with mean-reverting behavior, widely used in statistical physics, e.g., Langevin dynamics, and finance, e.g., interest-rate models. Unlike deterministic systems, OU trajectories exhibit intrinsic randomness due to process noise. The dynamics are given by the SDE
\begin{equation*}
    dX_t = -\theta X_t\,dt + \sigma\,dW_t,
\end{equation*}
where $W_t$ is standard Brownian motion, $\theta>0$ controls the mean-reversion rate, and $\sigma\ge 0$ controls the diffusion strength.
In our setting, the model observes $X_t$ and predicts $(\theta,\sigma)$.

\paragraph{Forced Duffing oscillator.}
The forced Duffing oscillator is a prototypical nonlinear driven system exhibiting multistability and, for suitable parameters, chaos. It is commonly used as a benchmark for nonlinear system identification because small parameter changes can substantially alter long-term behavior, especially in chaotic regimes. The dynamics are
\begin{equation*}
    \ddot{x} + \delta \dot{x} + \alpha x + \beta x^3 = \gamma\cos(\omega t),
\end{equation*}
or equivalently,
\begin{equation*}
    \dot{x}=v, \qquad \dot{v}=-\delta v - \alpha x - \beta x^3 + \gamma\cos(\omega t).
\end{equation*}
In our setting, the model observes $x(t)$ and predicts the scalar forcing amplitude $\gamma$, while holding $(\delta,\alpha,\beta,\omega)$ fixed.

For each class of dynamical systems, we generate 10,000 random training trajectories using a double-precision Runge--Kutta integrator with a sampling interval of $\tau = 0.01$. Note that even if for some training tasks, one only has one or two undetermined parameters, we also randomly sample the zeroth-order and first-order initial conditions, making the training trajectories heterogeneous enough. We then generate 1,000 random validation trajectories. For each training task, we use a one-layer model with a hidden dimension of $d = 64$ and a latent-state dimension of $n = 32$. All models are trained for $10$ epochs.

\clearpage

\section{Details of Experimemts in~\Cref{sec:contdata}}\label{app:sec4exp}

In this appendix, we present the details of the experiments in the main text that aim to understand the relationship between the continuity of a model and the continuity of a dataset.

For the synthetic dynamical systems dataset, we combine the four classes of dynamical systems discussed in~\cref{sec:contmodel}. The goal, however, is to predict the next sample in the trajectory, as opposed to unraveling the coefficients in the dynamical systems. The model returns its prediction of the embedded token, instead of the raw data. That is, given a sampled sequence
\[
    u(0), u(\tau), \ldots, u(\tau L)
\]
and its embedding
\[ 
    \mathbf{U}_\eta(0), \mathbf{U}_\eta(\tau), \ldots, \mathbf{U}_\eta(\tau L),
\]
the model outputs its prediction of $\mathbf{U}_\eta(\tau (L+1))$ instead of $u(\tau (L+1))$. We made this decision because quantization inevitably takes away some precision from the input data, making a high-precision prediction of $u(\tau (L+1))$ intrinsically harder than that of $\mathbf{U}_\eta(\tau (L+1))$.

We train an S4, an S6, and a Transformer for $10$ epochs. For every model, we set the hidden dimension to be $d = 128$ and the number of layers to be $4$. In addition, the Transformer has $8$ heads and the two SSMs have a state-space dimension of $n = 32$.

\begin{figure}[!htb]
    \centering
    \includegraphics[width=0.9\linewidth]{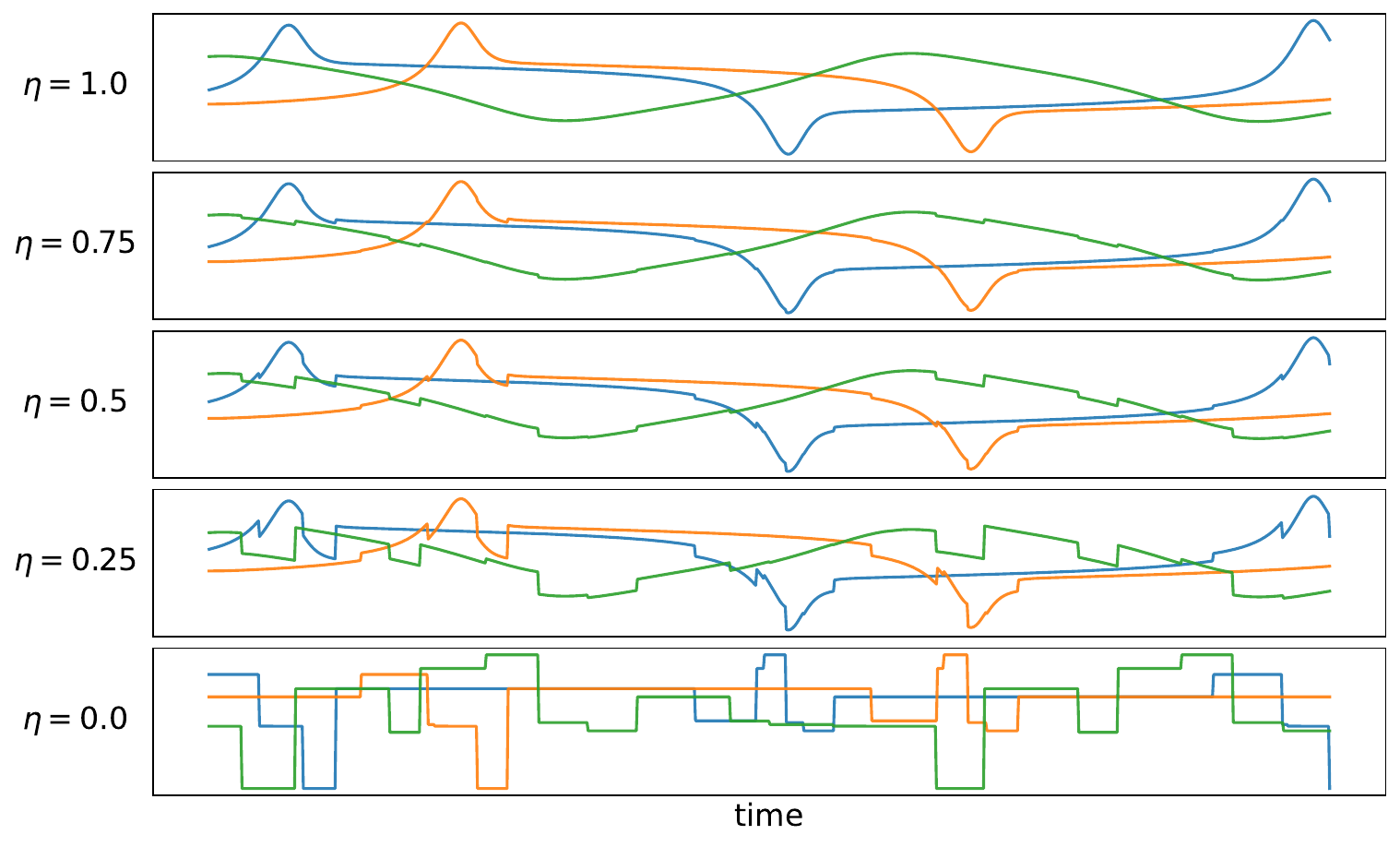}
    \caption{Visualization of three trajectories from a Van der Pol oscillator. Given a different $\eta$, the input is obtained by a weighted combination of a continuously embedded trajectory and a discretely embedded trajectory. As $\eta$ goes from $1$ to $0$, the level of the continuity of the task visually decreases.}
    \label{fig:trajectories}
\end{figure}

In~\Cref{fig:trajectories}, we present a visualization of three embedded trajectories $\mathbf{U}_\eta(t)$ from a Van der Pol oscillator, as $\eta$ ranges from $1$ to $0$. The visualization is done on the first channel of the $16$-dimensional input data $\mathbf{U}_\eta(t)$. Clearly, one can also visualize the data by performing a principal component analysis (PCA); however, since each $\mathbf{U}_\eta(t)$ is a sampled column of a fixed orthogonal matrix when $\eta = 0$ and $\mathbf{U}_\eta(\cdot)$ is a rank-one quasimatrix when $\eta = 1$, a PCA would help in neither cases.

Looking at the two metrics $\eta$ and $\mu$ that we adopt in~\cref{sec:contdata}, you may also wonder: are they consistent with each other? Clearly, given any general dataset, one can compute a $\mu$ but not an $\eta$, but for the dynamical systems dataset, one can still compute the $\mu$-metric to see if $\mu$ indeed positively correlates to $\eta$. In~\Cref{fig:eta_mu}, we present this analysis. We can see that for each lag parameter $t$, $\eta$ and $\mu_t$ are positively correlated. This further shows that our proposed $\mu$-metric is consistent with our intuition in the continuity of a dataset.

\begin{figure}[!htb]
    \centering
    \includegraphics[width=0.9\linewidth]{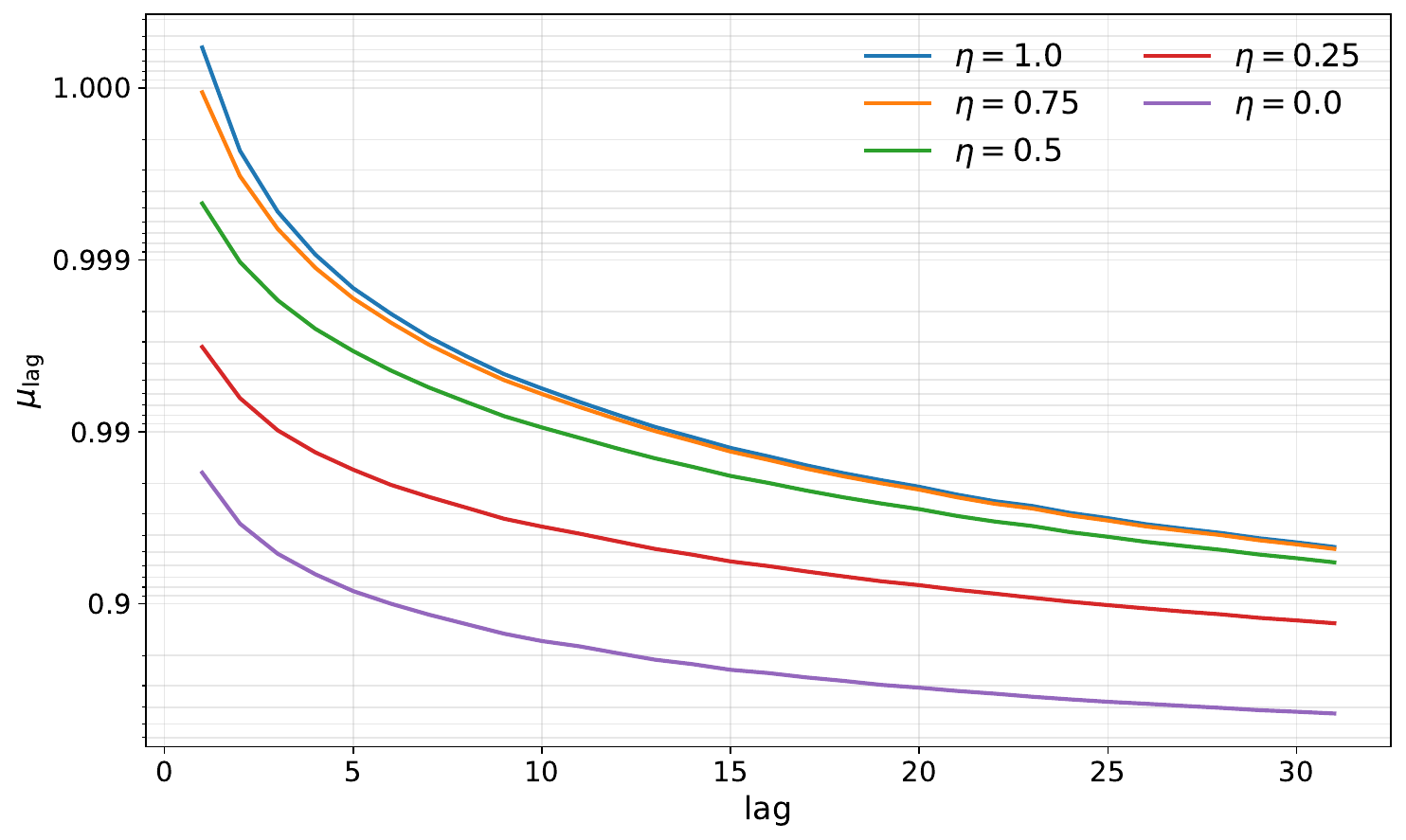}
    \caption{The metric $\mu_{\text{lag}}$ that we proposed in~\cref{sec:contdata} computed on the dynamical systems data embedded with a different level of continuity $\eta$. As $\eta$ decreases from $1$ to $0$, the metric $\mu_{\text{lag}}$ also decreases. This shows a consistency between the intuitive continuity captured by $\eta$ and the numerical continuity measured by our proposed $\mu_{\text{lag}}$ metric.}
    \label{fig:eta_mu}
\end{figure}

To define the normalized relative gap shown in~\Cref{tab:gap_heatmap}, given a model $\mathcal{M}$ and a continuity level $\eta$, let $e_\mathcal{M}^\eta$ denote the $\mathcal{L}_\infty$ forecasting error of $\mathcal{M}$ averaged over the test set. Let $\mathbb{M}$ denote the collection of all six models under consideration. We first define the relative gap of model $\mathcal{M}$ at continuity level $\eta$ by
\[
    \text{gap}_\mathcal{M}^\eta
    =
    \frac{e_\mathcal{M}^\eta - \min_{\mathcal{M}' \in \mathbb{M}} e_{\mathcal{M}'}^\eta}
    {\min_{\mathcal{M}' \in \mathbb{M}} e_{\mathcal{M}'}^\eta}.
\]
This quantity measures how far $\mathcal{M}$ is from the best-performing model at the same continuity level $\eta$. However, this metric still has one limitation when used to compare different architectures: each model comes with its own parameter count, architecture class, and inductive biases beyond continuity, so some models may be uniformly stronger or weaker across all values of $\eta$. As a result, the raw relative gap is appropriate for tracking how a fixed architecture responds to changing continuity, but it is less suitable for comparing the shape of this trend across architectures. To remove this architecture-dependent scale effect while preserving the relative variation of each model across continuity levels, we introduce the normalized gap
\[
    \overline{\text{gap}}_\mathcal{M}^\eta
    =
    \alpha_\mathcal{M}\,\text{gap}_\mathcal{M}^\eta,
    \qquad
    \alpha_\mathcal{M}
    =
    \frac{\frac{1}{|\mathbb{M}|}\sum_{\widetilde{\mathcal{M}} \in \mathbb{M}}
    \sum_{\eta'} \text{gap}_{\widetilde{\mathcal{M}}}^{\eta'}}
    {\sum_{\eta'} \text{gap}_\mathcal{M}^{\eta'}}.
\]
Equivalently, $\overline{\text{gap}}_\mathcal{M}^\eta$ is obtained by multiplying each row of relative gaps by a model-dependent constant so that all models have the same total gap mass across continuity levels. This normalization preserves the ratios between any two entries within the same row, and hence preserves how performance varies with continuity for each model, while making the rows more comparable across architectures.

Finally, we explain how the data in~\Cref{fig:continuoustasks} are obtained. For S4D, we directly collect data from the original S4D paper~\citep{gu2022parameterization}; for Transformers, experiments show that training a Transformer from scratch on LRA tasks does not lead to good performance, and we therefore adopt data induced by the pretraining method proposed in~\citet{amos2023never}. For S6,~\citet{yu2025block} shows that the performance on LRA can be significantly strengthened by initializing the weight $w_\Delta$ small, adopting a multihead-plus-bias structure, and using a complex parameterization. Many of these approaches can be viewed as a way to make the model more continuous; therefore, we do not adopt any of these methods. Instead, the data for S6 are obtained by maximizing the per-task performance reported in~\citet{alonso2025state},~\citet{sieber2024understanding},~\citet{fluderm_mlra}, and of the vanilla Mamba pretrained using the method proposed in~\citet{amos2023never}.

\clearpage

\section{Details of Experimemts in~\Cref{sec:accelerate}}\label{app:sec5exp}

In this appendix, we present the details of the experiments in the main text that aim to show how continuity can be leveraged to enhance the training efficiency of an S4D model on the LRA benchmark.

In~\Cref{fig:subsampledtraining}, we showed the times--accuracy plot for a vanilla S4D and a subsampled S4D trained on all six LRA benchmark tasks. There, we report the test accuracy of the best model, measured by
\begin{align*}
    &\text{test acc of the best model at epoch } K \\
    &\qquad= \text{test acc at epoch } \{\text{arg\,max}_{i \leq K} (\text{validation acc at epoch } i)\}.
\end{align*}
In~\Cref{fig:accepoch}, we further show the per-epoch test accuracy of the model, which gives us a more concrete understanding of how the stages transition. From~\Cref{fig:accepoch}, we note an important finding:
\begin{mdframed}
    For the more continuous tasks, such as Image, Pathfinder, and PathX, the loss of test accuracy is relatively low at each stage transition; for the more discrete tasks, such as ListOps, Retrieval, and Text, the loss of accuracy is sharp at each stage transition. This shows that for the continuous tasks, more information is inherited by the refined model at the start of each training stage.
\end{mdframed}

\begin{figure}[!htb]
    \centering
    \includegraphics[width=1\linewidth]{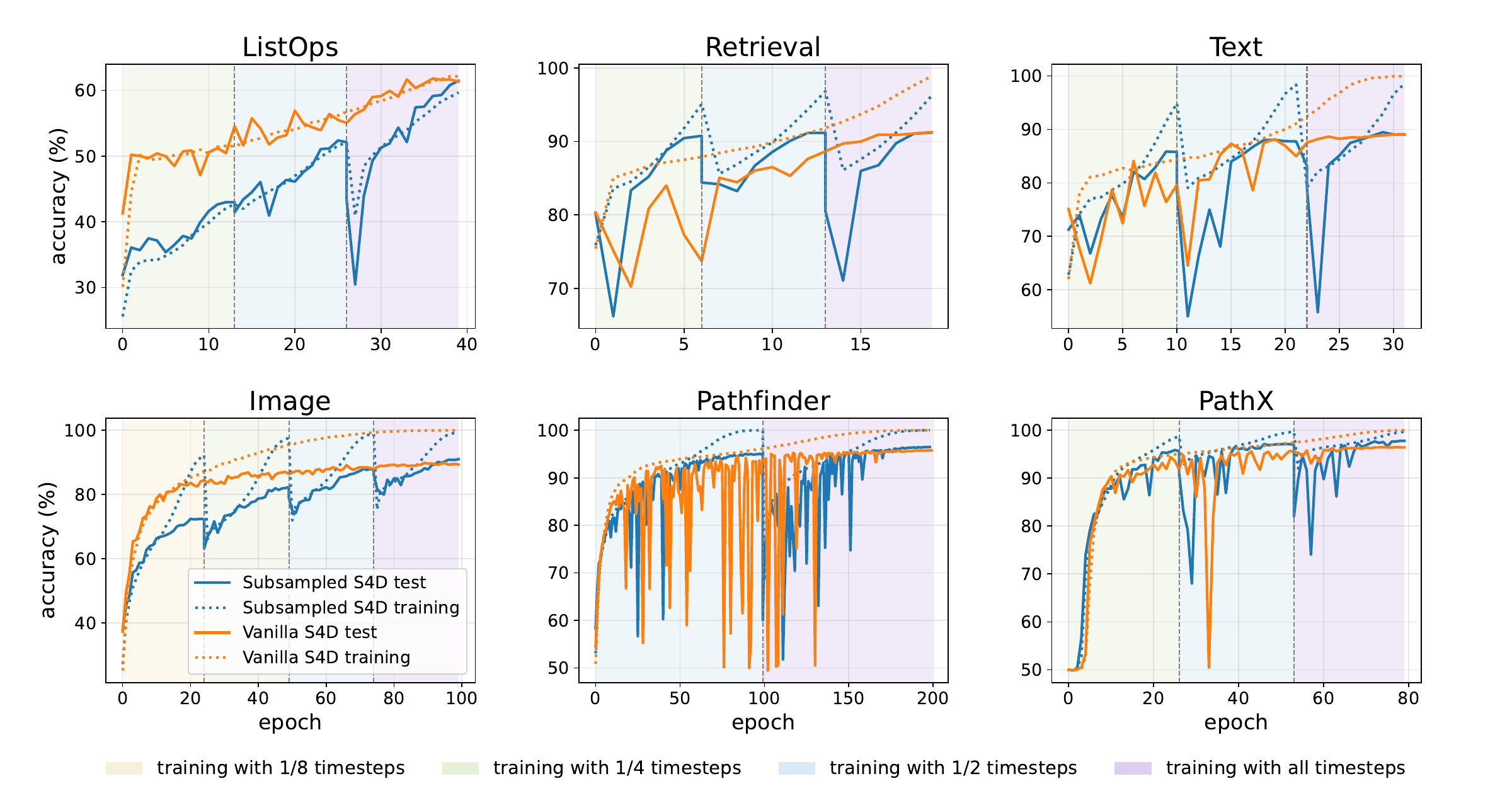}
    \caption{The per-epoch training and test accuracies of a naively trained S4D model and an S4D model trained with our subsampling method, respectively. We use a colorful background to indicate stages where models are trained on data with different temporal densities. Note that the background colors only apply to the models trained with the subsampling method, and that the sparser than timestamps are, the faster the training of the subsampled model is.}
    \label{fig:accepoch}
\end{figure}

For training the subsampled models, there are a few hyperparameters that we need to choose. We briefly explain them:
\begin{itemize}[leftmargin=*]
    \item \texttt{Initial Density}: This indicates the temporal density of the data we subsample at the beginning of the training. For example, if the full data has a length of $L = 1024$, and the training data we use in epoch zero has a length of $128$, then the initial density is $1/8$.
    \item \texttt{Subsampling Strategy}: This is a binary decision to make from ``indexing'' or ``pooling.'' Assuming that the initial density is $1/r$ If the subsampling is done by indexing, then given a sequence $\mathbf{u}_{0}, \mathbf{u}_1 \ldots, \mathbf{u}_{L-1}$, the subsampled sequence is $\mathbf{u}_0, \mathbf{u}_r, \ldots, \mathbf{u}_{\lfloor (L-1) / r \rfloor r}$. If the subsampling is done by pooling, then the subsampled sequence is $\mathbf{v}_0, \mathbf{v}_r, \ldots, \mathbf{v}_{\lfloor (L-1) / r \rfloor r}$, where
    \[
        \mathbf{v}_j = \frac{1}{r}\sum_{i=0}^{r-1} \mathbf{u}_{jr+i}.
    \]
\end{itemize}
Given an initial density, which we always choose to be an inverse power of $2$, we then evenly divide all training epochs into the resulting number of stages. While we only examine this way of division, other division strategies can be potentially interesting to explore.
The pooling strategy is especially useful when the underlying task is more discrete than continuous, for otherwise a subsampling by index may totally destroy the useful information in the data. In fact, among the six tasks in LRA, only the most continuous one (i.e., Image; see~\Cref{fig:continuoustasks}) can be trained using the indexing subsampling strategy. We report the hyperparameters that we use in~\Cref{tab:hyderparams}.

\begin{table}[!htb]
    \centering
    \caption{The hyperparameters \texttt{Initial Density} and \texttt{Subsampling Strategy} used in training the S4D models on Long-Range Arena tasks. The rest of the hyperparameters follow those reported in the original S4D paper~\citep{gu2022parameterization}.}
    \begin{tabular}{c|cccccc}
        \toprule
        Hyperparameter & ListOps & Retrieval & Text & Image & Pathfinder & PathX \\
        \midrule
        \texttt{Initial Density} & $1/4$ & $1/4$ & $1/4$ & $1/8$ & $1/2$ & $1/4$ \\
        \texttt{Subsampling Strategy} & pooling  & pooling & pooling & indexing & pooling & pooling \\
        \bottomrule
    \end{tabular}
    \label{tab:hyderparams}
\end{table}

Finally, to show that the subsampling method is helpful only when the underlying model is continuous, we show in~\Cref{fig:mambacifar} the results of applying the subsampling method on an S6 model trained on the Image task from LRA, in which case we see that most learned information is lost at each stage transition and the subsampling method does not help.

\begin{figure}[!htb]
    \centering
    \includegraphics[width=0.95\linewidth]{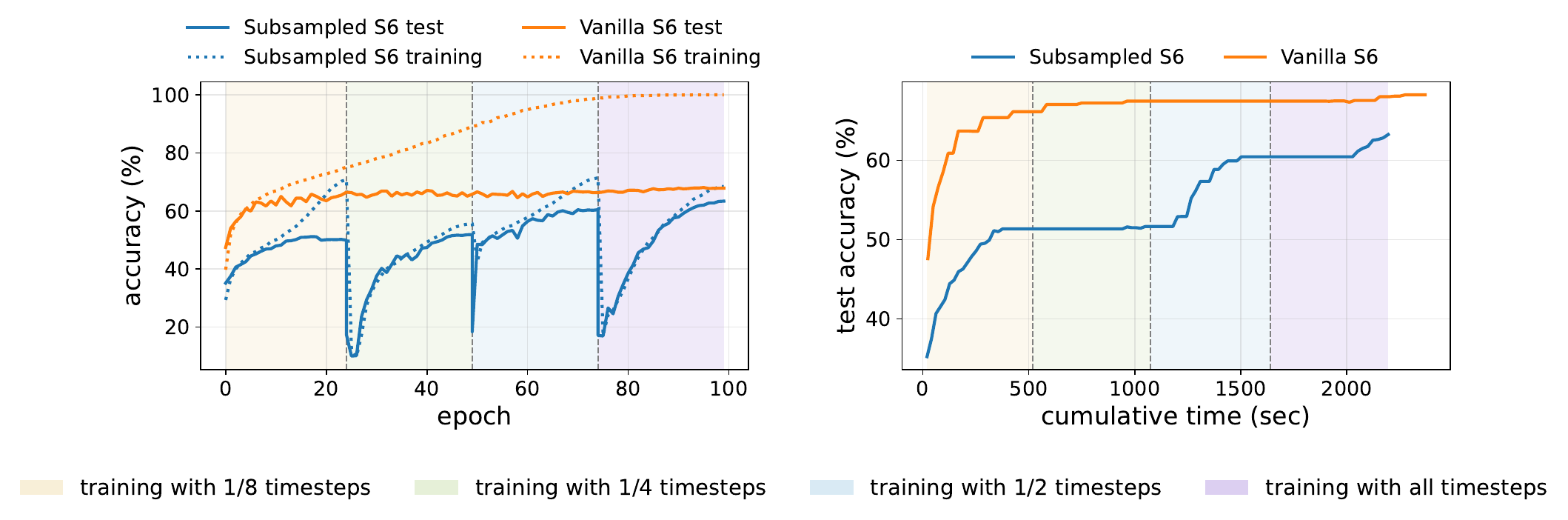}
    \caption{\textbf{Left:} The per-epoch training and test accuracies on the sCIFAR-10 task of a naively trained S6 model and an S6 model trained with our subsampling method, respectively. \textbf{Right:} Test accuracy plotted against cumulative training time on the sCIFAR-10 task. For more details, see the captions of~\Cref{fig:subsampledtraining,fig:accepoch}.}
    \label{fig:mambacifar}
\end{figure}

\clearpage

\section{Limitations}\label{app:limitation}

Our work takes a first step toward formalizing continuity as an inductive bias in sequential modeling, but several directions remain open. On the theory side, our main analysis is developed for SSM-like models, where the underlying continuous-time formulation provides a natural starting point; extending comparable guarantees to broader model classes, especially Transformers and hybrid architectures, would be an important next step. On the empirical side, although we study several representative model families and tasks, we do not attempt an exhaustive survey of the rapidly growing space of sequential architectures, and the paper should therefore be viewed as identifying and isolating one important bias rather than as a complete guide for model selection. Finally, while our experiments show a clear empirical alignment between model continuity and task continuity, a deeper theoretical explanation for why more continuous models are better suited to more continuous tasks remains open. We believe each of these directions is substantial enough to support further dedicated study.

\end{document}